\documentclass{article}

\usepackage{graphicx}
\usepackage{subfigure}
\usepackage{booktabs} 
\usepackage{enumitem}
\usepackage{multirow}
\usepackage[dvipsnames]{xcolor}
\usepackage[pages=some,scale=1,angle=0,opacity=0.7]{background}

\usepackage[tracking=true]{microtype}

\PassOptionsToPackage{hyphens}{url}
\usepackage{hyperref}
\hypersetup{
    colorlinks=true,
    pdftitle={A Watermark for Large Language Models},
    pdfpagemode=FullScreen,
    }

\usepackage[accepted]{shmicml2023}

\usepackage{amsmath}
\usepackage{amssymb}
\usepackage{mathtools}
\usepackage{amsthm}

\usepackage[capitalize,noabbrev,nameinlink]{cleveref}

\theoremstyle{plain}
\newtheorem{theorem}{Theorem}[section]

\newtheorem{lemma}[theorem]{Lemma}

\theoremstyle{definition}
\newtheorem{definition}[theorem]{Definition}

\theoremstyle{remark}

\usepackage[textsize=tiny]{todonotes}

\usepackage[raggedrightboxes]{ragged2e}
\usepackage{float}
\usepackage{xcolor}
\usepackage{soulpos}
\usepackage[most]{tcolorbox}

\usepackage{mdframed}

\definecolor{customgray}{rgb}{0.9, 0.9, 0.9}
\definecolor{lightgreen}{rgb}{0.83, 0.96, 0.72}
\definecolor{lightyellow}{rgb}{0.94, 0.86, 0.51}
\definecolor{lightred}{rgb}{0.99, 0.66, 0.66}
\ulposdef{\hlgray}[xoffset=1pt]{\mbox{\color{customgray}\rule[-.7ex]{\ulwidth}{2.5ex}}}
\ulposdef{\hlgreen}[xoffset=1pt]{\mbox{\color{lightgreen}\rule[-.7ex]{\ulwidth}{2.5ex}}}
\ulposdef{\hlyel}[xoffset=1pt]{\mbox{\color{lightyellow}\rule[-.7ex]{\ulwidth}{2.5ex}}}
\ulposdef{\hlred}[xoffset=1pt]{\mbox{\color{lightred}\rule[-.7ex]{\ulwidth}{2.5ex}}}

\DeclareMathOperator*{\expect}{\mathbb{E}}
\DeclareMathOperator*{\prob}{\mathbb{P}}
\DeclareMathOperator*{\key}{{\mathcal{K}}}

\newcommand{\vocab}{\mathcal{V}}

\definecolor{teaserblue}{RGB}{242, 242, 255}
\icmltitlerunning{A Watermark for Large Language Models. Page \thepage\ of \pageref{lastpagemaintext}.}

\tcbset{
    frame code={}
    center title,
    left=0pt,
    right=0pt,
    top=0pt,
    bottom=0pt,
    colframe=blue!5, 
    colback=blue!5,
    enlarge left by=0mm,
    boxsep=5pt,
    arc=0pt,outer arc=0pt,
    }

\begin{document}

\twocolumn[

\icmltitle{A Watermark for Large Language Models}

\icmlsetsymbol{equal}{*}

\begin{icmlauthorlist}
\icmlauthor{John Kirchenbauer}{equal}
\icmlauthor{Jonas Geiping}{equal}
\icmlauthor{Yuxin Wen}{}
\icmlauthor{Jonathan Katz}{}
\icmlauthor{Ian Miers}{}
\icmlauthor{Tom Goldstein}{}

\end{icmlauthorlist}

\begin{center}
\textbf{University of Maryland} 
\end{center}

\icmlcorrespondingauthor{John Kirchenbauer}{jkirchen@umd.edu}

\icmlkeywords{Machine Learning, LLMs, Watermark, Language Model, Natural Language Processing, Generative AI}

\vskip 0.3in
]

\printAffiliationsAndNotice{\icmlEqualContribution. Code and demo are available at \href{https://github.com/jwkirchenbauer/lm-watermarking}{\texttt{github.com/jwkirchenbauer/lm-watermarking}}}

\renewenvironment{abstract}
 {
  \begin{center}
  \bfseries \abstractname\vspace{-.5em}\vspace{0pt}
  \end{center}
  \list{}{%
    \setlength{\leftmargin}{0mm}
    \setlength{\rightmargin}{\leftmargin}%
  }%
  \item\relax}
 {\endlist}
\begin{abstract}
\vspace{-.2cm}
\looseness -1 Potential harms of large language models can be mitigated by {\em watermarking} model output, i.e., embedding signals into generated text that are invisible to humans but algorithmically detectable from a short span of tokens.  
We propose a watermarking framework for proprietary language models. The watermark can be embedded with negligible impact on text quality, and can be detected using an efficient open-source algorithm without access to the language model API or parameters.  The watermark works by selecting a randomized set of  ``green'' tokens before a word is generated, and then softly promoting use of green tokens during sampling.  
We propose a statistical test for detecting the watermark with interpretable p-values, and derive an information-theoretic framework for analyzing the sensitivity of the watermark. We test the watermark using a multi-billion parameter model from the Open Pretrained Transformer (OPT) family, and discuss robustness and security.
\end{abstract}
\vspace{-0.5cm}

\BgThispage
\backgroundsetup{
  contents={\hspace{9.2cm}   \includegraphics[height=17cm]{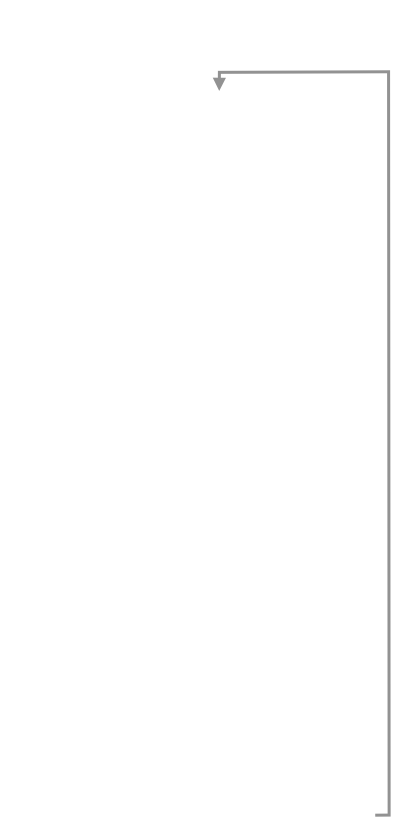}
  }
  }
  
\section{Introduction}
\label{sec:intro_rules}
Large language models (LLMs), such as the recently developed ChatGPT, can write documents, create executable code, and answer questions, often with human-like capabilities \citep{schulman_chatgpt_2022}.  As these systems become more pervasive, there is increasing risk that they may be used for malicious purposes \citep{bergman_guiding_2022, mirsky_threat_2023}.  These include social engineering and election manipulation campaigns that exploit automated bots on social media platforms, creation of fake news and web content, and use of AI systems for cheating on academic writing and coding assignments. Furthermore, the proliferation of synthetic data on the web complicates future dataset creation efforts, as synthetic data is often inferior to human content and must be detected and excluded before model training \citep{radford_robust_2022}.
For many reasons, the ability to detect and audit the usage of machine-generated text becomes a key principle of harm reduction  for large language models~\citep{bender_dangers_2021,crothers_machine_2022,grinbaum_ethical_2022}.

\begin{figure}[t]
\vskip 0.2in
\begin{center}
\centerline{\includegraphics[width=\columnwidth]{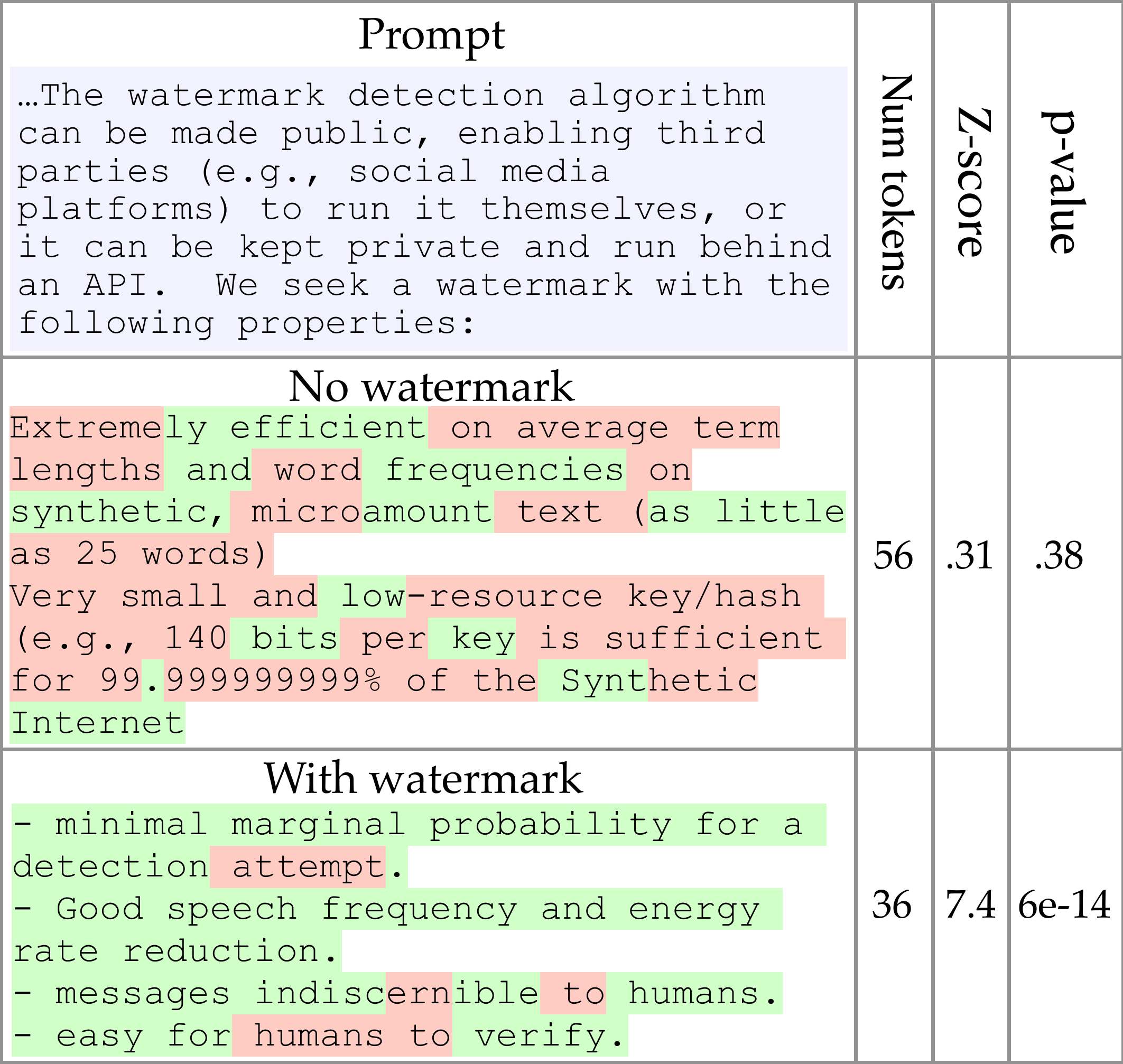}}
\vspace{-2mm}
\caption{Outputs of a language model, both with and without the application of a watermark.  The watermarked text, if written by a human, is expected to contain 9 ``green'' tokens, yet it contains 28. The probability of this happening by random chance is $\approx 6\times 10^{-14}$, leaving us {\em extremely} certain that this text is machine generated. 
Words are marked with their respective colors. The model is OPT-6.7B using multinomial sampling. Watermark parameters are $\gamma,\delta=(0.25,2)$. The prompt is the whole paragraph marked in blue below.}
\label{fig:teaser}
\vspace{-0.75cm}
\end{center}
\end{figure}

\begin{mdframed}[backgroundcolor=teaserblue,hidealllines=true]
In this work, we study \emph{watermarking} of language model output. A watermark is a hidden pattern in text that is imperceptible to humans, while making the text algorithmically identifiable as synthetic. We propose an efficient watermark that makes synthetic text  detectable from short spans of tokens (as few as 25 tokens),  while false-positives (where human text is marked as machine-generated)  are statistically improbable.  The watermark detection algorithm can be made public,   enabling third 
parties (e.g., social media  platforms) to run it themselves, or it can be kept private and run behind an API. We seek
a watermark with the following properties:
\end{mdframed}

\begin{itemize}
\item The watermark can be algorithmically detected without any knowledge of the model parameters or access to the language model API.  This property allows the detection algorithm to be open sourced even when the model is not.  This also makes detection cheap and fast because the LLM does not need to be loaded or run.
\item Watermarked text can be generated using a standard language model without re-training. 
\item The watermark is detectable from only a contiguous portion of the generated text. This way, the watermark remains detectable when only a slice of the generation is used to create a larger document.
\item The watermark cannot be removed without modifying a significant fraction of the generated tokens.
\item We can compute a rigorous statistical measure of confidence that the watermark has been detected.
\end{itemize}

\subsection{Notation \& Language model basics}

Language models have a ``vocabulary'' $\vocab$ containing words or word fragments known as ``tokens.'' Typical vocabularies contain $|\vocab| = 50,000$ tokens or more \citep{radford_language_2019,liu_roberta_2019}. 
Consider a sequence of $T$ tokens $\{s^{(t)}\}\in \vocab^T$. Entries with negative indices, $s^{(-N_p)},\cdots, s^{(-1)}$, represent a ``prompt'' of length $N_p$ and $s^{(0)},\cdots,s^{(T)}$ are tokens generated by an AI system in response to the prompt.

A {\em language model} (LM) for next word prediction is a function $f$, often parameterized by a neural network, that accepts as input a sequence of known tokens $s^{(-N_p)}, \cdots, s^{(t-1)}$, which contains a prompt and the first $t-1$ tokens already produced by the language model, and then outputs a vector of $|V|$ logits, one for each word in the vocabulary. These logits are then passed through a softmax operator to convert them into a discrete probability distribution over the vocabulary.
The next token at position $t$ is then sampled from this distribution using either standard multinomial sampling, or \textit{greedy} sampling (greedy decoding) of the single most likely next token. Additionally, a procedure such as {\em beam search} can be employed to consider multiple possible sequences before selecting the one with the overall highest score.

\subsection{A caveat: The difficulty of watermarking low-entropy sequences} \label{entropyproblem}
Consider the following two sequences of tokens, with prompts in red:
\begin{center}
\vspace{-2mm}
\texttt{{\color{red}The\!\! quick\!\! brown}\!\! fox\!\! jumps\! over\! the\!\! lazy\!\! dog}\\
\texttt{{\color{red}for(i=}0;i<n;i++) sum+=array[i]}
\vspace{-2mm}
\end{center}
Were they produced by a human or by a language model?
Determining this is fundamentally hard because these sequences have low entropy; the first few tokens strongly determine the following tokens.  

Low entropy text creates two problems for watermarking.  First, both humans and machines provide similar if not identical completions for low entropy prompts, making it impossible to discern between them.  Second, it is difficult to watermark low entropy text, as any changes to the choice of tokens may result in high perplexity, unexpected tokens that degrade the quality of the text. 
Later, we rigorously define sentence entropy, and analyze its impact on watermark detection.

\section{A simple proof of concept}
We start out by describing a simple ``hard'' red list watermark in \cref{hard} that is easy to analyze, easy to detect and hard to remove. The simplicity of this approach comes at the cost of poor generation quality on low entropy sequences. We will discuss more sophisticated strategies later.

\begin{algorithm}[h]
   \caption{Text Generation with Hard Red List}
   \label{hard}
\begin{algorithmic}
\STATE \textbf{Input:} prompt, $s^{(-N_p)}\cdots s^{(-1)}$
\FOR{$t=0,1,\cdots$}
 \STATE 
 \begin{enumerate}
\item Apply the language model to prior tokens $s^{(-N_p)}\cdots s^{(t-1)}$ to get a probability vector $p^{(t)}$ over the vocabulary.
\item Compute a hash of token $s^{(t-1)},$ and use it to seed a random number generator.
\item Using this seed, randomly partition the vocabulary into a ``green list'' $G$ and a ``red list'' $R$ of equal size. 
\item Sample $s^{(t)}$ from $G$ , never generating any token in the red list.
\end{enumerate}
\ENDFOR
\end{algorithmic}
\end{algorithm}
The method works by generating a pseudo-random red list of tokens that are barred from appearing as $s^{(t)}.$   The red list generator is seeded with the prior token $s^{(t-1)}$, enabling the red list to be reproduced later without access to the entire generated sequence.

\textbf{Detecting the watermark.}
While producing watermarked text requires access to the language model, detecting the watermark does not. A third party with knowledge of the hash function and random number generator can re-produce the red list for each token and count how many times the red list rule is violated.
We can detect the watermark by testing the following null hypothesis, 
\begin{align} \label{null}
\begin{split}
\text{\em $H_0$: The text sequence is generated with} \\ \text{\em no knowledge of the red list rule.} 
\end{split}
\end{align}

Because the red list is chosen at random, a natural writer is expected to violate the red list rule with half of their tokens, while the watermarked model produces no violations. The probability that a natural source produces $T$ tokens without violating the red list rule is only $1/2^T,$ which is vanishingly small even for short text fragments with a dozen words.  This enables detection of the watermark (rejection of $H_0$) for, e.g., a synthetic tweet.   

A more robust detection approach uses a {\em one proportion z-test} to evaluate the null hypothesis. If the null hypothesis is true, then the number of green list tokens, denoted $|s|_G,$ has expected value $T/2$ and variance $T/4.$ The $z$-statistic for this test is  
\begin{align} \label{zformula}
z = 2(|s|_G - T/2)/\sqrt{T}.
\end{align}
We reject the null hypothesis and detect the watermark if $z$ is above a chosen threshold. Suppose we choose to reject the null hypothesis if $z>4.$  In this case, the probability of a false positive is $3\times 10^{-5},$ which is the one-sided p-value corresponding to $z>4.$ At the same time, we will detect any watermarked sequence with 16 or more tokens (the minimum value of $T$ that produces $z=4$ when $|s|_G$=T).

\textbf{How hard is it to remove the watermark?} \label{remove}
The use of the one proportion z-test makes removal of the watermark difficult.  Consider the case of a watermarked sequence of length $T=1000$. Suppose an adversary modifies 200 tokens in the sequence to add red list words and scrub the watermark. A modified token at position $t$ can violate the red list rule at position $t$.  Furthermore, the value of $s_t$ determines the red list for token $s_{t+1},$ and a maximally adversarial choice of $s_t$ will put $s_{t+1}$ in violation of the red list rule as well.  For this reason, 200 token flips can create at most 400 violations of the red list rule.  Unfortunately for the attacker, this maximally adversarial sequence with 600 remaining green list tokens still produces a z-statistic of $2(600-1000/2)/\sqrt{1000} \approx 6.3,$ and a p-value of $\approx 10^{-10},$ leaving the watermark readily detectable with extremely high confidence. In general, removing the watermark of a long sequence requires modifying roughly one quarter of the tokens or more. 

\looseness -1 Note the analysis above assumes the attacker has complete knowledge of the watermark, and each selected token is maximally adversarial (which likely has a negative impact on quality). Without knowledge of the watermark algorithm, each flipped token has only a 50\% chance of being in the red list, as does the adjacent token.  In this case, the attacker above only creates 200 red list words (in expectation) by modifying 200 tokens.  Methods for keeping the watermark algorithm secret but available via API are discussed in \cref{cryptosec}.

\textbf{Drawbacks of the hard red list rule.}
The hard red list rule handles low entropy sequences in a simple way; it prevents the language model from producing them. For example, the token ``Barack'' is almost deterministically followed by ``Obama'' in many text datasets, yet ``Obama'' may be disallowed by the red list.

A better behavior is to use a ``soft'' watermarking rule that is only active for high-entropy text that can be imperceptibly watermarked. As long as low-entropy sequences are wrapped inside a passage with enough total entropy, the passage will still easily trigger a watermark detector, solving the problem described in \cref{entropyproblem}. 
Further, one can combine the watermark with a beam search decoder that ``irons-in'' the watermark. By searching the hypothesis space of likely token sequences, 
candidates sequences with a high density of tokens in the green list are found, resulting in a high strength watermark with minimal perplexity cost.

\section{A more sophisticated watermark}\label{sec:soft-watermark}
We now discuss the ``soft'' watermark that promotes the use of the green list for high entropy tokens when many good choices are available, while having little impact on the choice of low-entropy tokens that are nearly deterministic. 

To derive this watermark, we examine what happens in the language model just before it produces a probability vector. The last layer of the language model outputs a vector of logits $l^{(t)}$.  These logits get converted into a probability vector $p^{(t)}$ using the softmax operator
 $$p^{(t)}_k = \exp(l^{(t)}_k)/\sum_i \exp(l^{(t)}_i).$$
Rather than strictly prohibiting the red list tokens, Algorithm \ref{soft} adds a constant $\delta$ to the logits of the green list tokens.

\begin{algorithm}[ht]
   \caption{Text Generation with Soft Red List}
   \label{soft}
\begin{algorithmic}
\STATE \textbf{Input:} prompt, $s^{(-N_p)}\cdots s^{(-1)}$

\hspace{10.5mm}green list size, $\gamma \in (0,1)$ 

\hspace{10.5mm}hardness parameter, $\delta > 0$
\FOR{$t=0,1,\cdots$}
 \STATE 
 \begin{enumerate}
\item Apply the language model to prior tokens $s^{(-N_p)}\cdots s^{(t-1)}$ to get a logit vector $l^{(t)}$ over the vocabulary.
\item Compute a hash of token $s^{(t-1)},$ and use it to seed a random number generator.
\item Using this random number generator, randomly partition  the vocabulary into a ``green list'' $G$ of size $\gamma |V|,$ and a ``red list'' $R$ of size $(1-\gamma)|V|$. 
\item Add $\delta$ to each green list logit.  Apply the softmax operator to these modified logits to get a probability distribution over the vocabulary.
 $$\hspace{-4mm}\hat p^{(t)}_k =
  \begin{cases}  \frac{ \exp(l^{(t)}_k+\delta)}{\sum_{i\in R} \exp(l^{(t)}_i)+\sum_{i\in G} \exp(l^{(t)}_i+\delta)}, \quad k\in G\\
  \frac{ \exp(l^{(t)}_k)}{\sum_{i\in R} \exp(l^{(t)}_i)+\sum_{i\in G} \exp(l^{(t)}_i+\delta)}, \quad k\in R. \label{logitboost}
  \end{cases}
$$
\item Sample the next token, $s^{(t)},$ using the watermarked distribution $\hat p^{(t)}.$  
\end{enumerate}
\ENDFOR
\end{algorithmic}
\end{algorithm}

\looseness -1 The soft red list rule adaptively enforces the watermark in situations where doing so will have little impact on quality, while almost ignoring the watermark rule in the low entropy case where there is a clear and unique choice of the ``best'' word.  A highly likely word with $p^{(t)}_k \approx 1$ has a much larger logit than other candidates, and this will remain the largest regardless of whether it is in the red list.  But when the entropy is high, there are many comparably large logits to choose from, and the $\delta$ rule has a large impact on the sampling distribution, strongly biasing the output towards the green list.

\subsection{Detecting the soft watermark}

The process for detecting the soft watermark is identical to that for the hard watermark.  We assume the null hypothesis \eqref{null} and compute a z-statistic using Equation \eqref{zformula}.  We reject the null hypothesis and detect the watermark if $z$ is greater than a threshold.  For arbitrary $\gamma$ we have 
\begin{align} \label{zformula_generalized}
z = (|s|_G - \gamma T)/\sqrt{T\gamma(1-\gamma)}.
\end{align}

\looseness -1 Consider again the case in which we detect the watermark for $z>4.$  Just like in the case of the hard watermark, we get false positives with rate $3\times 10^{-5}.$  In the case of the hard watermark, we could detect any watermarked sequence of length 16 tokens or more, regardless of the properties of the text.  However, in the case of the soft watermark our ability to detect synthetic text depends on the entropy of the sequence. High entropy sequences are detected with relatively few tokens, while low entropy sequences require more tokens for detection.   
Below, we rigorously analyze the detection sensitivity of the soft watermark, and its dependence on entropy.  
 
\section{Analysis of the soft watermark}
In this section, we examine the expected number of green list tokens used by a watermarked language model and analyze the dependence of this quantity on the entropy of a generated text fragment.  Our analysis assumes the red list is sampled uniformly at {\em random}. 
This is a deviation from the method used in practice, which generates red lists using a {\em pseudo-random} number generator seeded with previous tokens. The consequences of pseudo-random sampling are explored in Section \ref{cryptosec}.  We analyze the case in which text is generated by multinomial random sampling.  In our experiments, we consider two more sampling schemes, greedy decoding and beam search.

We need a definition of entropy that is appropriate for our analysis.  The strength of our watermark is weak when the distribution over tokens has a large ``spike'' concentrated on one or several tokens. We define the following type of entropy to quantify this phenomenon.
\begin{definition}
Given a discrete probability vector $p$ and a scalar $z,$ we define the {\em spike entropy} of $p$ with modulus $z$ as
 $$S(p,z) = \sum_k \frac{p_k}{1+zp_k}.$$
\end{definition}
Like the classical Shannon entropy, the spike entropy is a measure of how spread out a distribution is; The spike entropy assumes its minimal value of $\frac{1}{1+z}$  when the entire mass of $p$ is concentrated at a single location, and its maximal value of $\frac{N}{N+z}$ when the mass of $p$ is uniformly distributed. 
For large $z$, the value of $\frac{p_k}{1+zp_k} \approx 1/z$ when $p_k>1/z$ and $\approx 0$ for $p_k < 1/z.$ For this reason, one can interpret the spike entropy as a softened measure of the number of entries in $p$ greater than  $1/z.$

The following theorem predicts the number of green list tokens that appear in a sequence with the watermark.

\begin{theorem} \label{maintheorem}
Consider watermarked text sequences of $T$ tokens. Each sequence is produced by sequentially sampling a raw probability vector $p^{(t)}$ from the language model, sampling a random green list of size $\gamma N$, and boosting the green list logits by $\delta$ using Equation \ref{logitboost} before sampling each token.  Define $\alpha = \exp(\delta),$ and let $|s|_G$ denote the number of green list tokens in sequence $s.$

If a randomly generated watermarked sequence has average spike entropy at least $S^\star,$ i.e.,
   $$ \frac{1}{T} \sum_t S\left(p^{(t)},\frac{(1-\gamma)(\alpha - 1)}{ 1+(\alpha-1)\gamma}\right) \ge S^\star,$$
then the number of green list tokens in the sequence has expected value at least
   $$ \expect |s|_G  \ge \frac{\gamma\alpha T}{ 1+(\alpha-1)\gamma} S^\star,$$
 Furthermore, the number of green list tokens has variance at most
  $$ \text{Var}\,\, |s|_G \le T \frac{\gamma\alpha S^\star}{ 1+(\alpha-1)\gamma}  \left(1-\frac{\gamma\alpha S^\star}{ 1+(\alpha-1)\gamma} \right).$$
  If we have chosen $\gamma\ge .5,$ then we can use the strictly looser but simpler bound 
  $$ \text{Var}\,\, |s|_G  \le T\gamma (1-\gamma).$$
\end{theorem}

\noindent \textbf{Remark.}  It may seem like there are a lot of messy constants floating around in this bound. However, when we choose $\gamma=\frac{1}{2}$ and $\delta = \ln(2) \approx 0.7,$ this bound simplifies to
$$
\expect |s|_G  \ge \frac{2}{3}T S^\star, \,\,\, \text{Var}\,\, |s|_G \le \frac{2}{3}T  S^\star   \left(1- \frac{2}{3} S^\star  \right)
$$
where $S^\star$ is a bound on spike entropy with modulus 1/3.
If we study the ``hard'' red list rules by choosing $\gamma=\frac{1}{2}$ and letting $\delta \to \infty,$ we have 
$$
\expect |s|_G  \ge T S^\star, \,\,\, \text{Var}\,\, |s|_G \le T S^\star   \left(1- S^\star  \right)
$$
where $S^\star$ is a bound on spike entropy with modulus 1. 

\subsection{Sensitivity of the watermark test}\label{sec:sensitivity}
The sensitivity of the soft watermark can be computed using standard type-II error analysis. For illustrative purposes, we estimate the type-II (false negative) error rate of a soft watermark with $\gamma=.5$ and $\delta=2.$ We assume 200 tokens are generated using OPT-1.3B \citep{zhang2022opt} using prompts from the C4 dataset's RealNewsLike subset \citep{2019t5}. We also assume a detection threshold of $z=4$ (which occurs at $\sim128.2/100$ tokens) which gives us a type-I error (false positive) rate of $3\times 10^{-5}$. 

\textbf{Theoretical bound.} 
 Our generations have an average spike entropy per sample of $S=0.807$ over $\sim500$ generations. \Cref{maintheorem} says that the expected number of green list tokens per generation is {\em at least} $142.2$. Indeed, the empirical average is $159.5$. For sequences with entropy equal to the mean ($S=0.807$) we get $\sigma\le 6.41$ tokens, and
98.6\% sensitivity (1.4\% type-II error rate), using a standard Gaussian approximation for the green list count. Note, this is a {\em lower} bound on the sensitivity for this particular entropy. If we use the true empirical mean of $159.5$ rather than the theoretical bound, we get a $5.3\times 10^{-7}$ type-II error rate, a realistic approximation but not a rigorous lower bound. 

\textbf{Empirical sensitivity.} Empirically, $98.4\%$ of generations are detected at the $z=4$ ($128$ token) threshold when multinomial sampling is used.  When $4$-way beam search over a greedy decoding is used, we get $99.6\%$ empirical sensitivity. Unlike the theoretical bounds, these are computed over all generations, which have the same length but vary in their individual entropies. Here, the primary source of type-II errors is low entropy sequences, as calculations above show that we expect a very low error rate when the entropy lies near the mean. To validate this, we examine the subset of 375/500 generations that have spike entropy above the $25$th percentile, of which we detect $100\%$ of generations at the $z=4$ threshold.

\textbf{What do failure cases look like?} We display typical success and failure cases for the watermark in \cref{tab:demo-examples}.  We observe that low-entropy (undetectable) sequences typically involve data memorization; the model regurgitates a copy (or near copy) of human-written text which is therefore not detectable as machine-written. A detailed exploration of model accuracy is presented in Section \ref{sec:experiments}, with more generation examples provided in \cref{sec:sample-outputs}.

\textbf{Evaluating Repetitive Text.}
A subtlety of the proposed approach is that tokens in the green list are only pseudo-random, and $n$-grams of text that are repeated will always be scored in the same manner.
Assume a $2$-gram, such as ``Barack Obama'' happens to green-list ``Obama''. Repetitive usage of this $2$-gram would result in a higher than expected number of green tokens.
In a worst-case scenario, human-generated text with a high number of repetitions of this 2-gram may be erroneously flagged as machine-generated. 

\looseness -1 Two remedies are possible: The first is to simply increase the length $h$ of the PRNG function, thereby increasing the variability of the green-listed words, as larger $(h+1)$-grams are much less likely to be repeated. A better remedy (possibly used in conjunction with the first) is not to count repeated $n$-grams when checking for the watermark. In the example above, the 2-gram ``Barack Obama'' would be counted on its first occurrence, and then subsequently ignored when it appears again; it is counted as neither green nor red, and the token counter $T$ is not incremented.

In addition to preventing false positives, skipping repeated $n$-grams can also make the detector more sensitive. A repeated $n$-gram is likely to be low-entropy, and so it will not contribute to the strength of the watermark.  By excluding these from the count, we keep the green list fraction high and maintain high sensitivity.

\begin{table*}[t]
\tiny
\begin{tabular}{p{3.1cm}|p{3.1cm}|p{3.1cm}|p{3.1cm}|p{0.28cm}|p{0.44cm}|p{0.3cm}|p{0.3cm}}
\toprule
prompt & real completion &  no watermark (NW) &  watermarked (W) & $S$ & (W) $z$ & (NW) PPL & (W) PPL \\
\midrule
 ...tled out of court and publicly reconciled.\textbackslash nIn the ’80s the band’s popularity waned in the United States but remained strong abroad. Robin released three solo albums, with limited success. The Bee Gees &   returned with some moderate hits in the late 1990s and were inducted into the Rock and Roll Hall of Fame in 1997. With his brothers, Mr. Gibb won six Grammys.\textbackslash nIn addition to his wife and his brother [...continues] &   continued to tour, and Barry became a television producer.\textbackslash nBut in the early ’90s, the Bee Gees’ popularity remained high. They scored a hit with “Don’t Stop Believing” in 1990, and in 1992 the Bee Ge[...continues] &  ’ 1990 album, “Spirits of the Century,” was a mixed critical and commercial success.\textbackslash nWhen the brothers were nominated for a Grammy Award in 1990, Mr. Gibb’s “You Should Be Dancing” and “Massachusetts,[...continues] &  0.68 &       \hlgreen{12.73} &       3.15 &      1.93 \\\midrule
 ... logged into their Google account and have verified profiles that match queries for the site.\textbackslash nGoogle's John Mueller said there is no ranking benefit in using different Google Search Console and Google &    Analytics accounts for each individual web site you manage. The topic came up before, as long as you are not spamming Google - there also is no down side to using the same accounts across multiple we[...continues] &   Analytics data to calculate ranking sentiment.\textbackslash nGoogle have a new feature called Keyword Difficulty Analysis that shows you the keywords your competitors are ranking for. It shows the demand curve as [...continues] & + accounts to see different numbers.\textbackslash nGoogle also released their own great blog post on the news algorithm. They give lots of great advice to help your site do better.\textbackslash nFinally, at the end of September [...continues] &  0.77 &       \hlgreen{11.88} &       9.24 &      4.10 \\\midrule
 ...1 on the season with a 2.76 GAA and.906 save percentage.\textbackslash nMar. 11 5:28 PM PT6:28 PM MT7:28 PM CT8:28 PM ET0:28 GMT8:28 5:28 PM MST6:28 PM CST7:28 PM EST4:28 UAE (+1)20:28 ET21:28 BRT - Sergei Bobrovsky &   stopped 18 of 19 shots in the Blue Jackets' 2-0 loss to the Islanders on Monday. The defeat dropped him to 29-22-1 on the season with a 2.74 GAA and.907 save percentage.\textbackslash nMar. 9 5:41 PM PT6:41 PM MT7:[...continues] &   stopped 28 of 30 shots in the Blue Jackets' 3-2 shootout loss to the Islanders on Sunday. The loss dropped him to 30-22-1 on the season with a 2.76 GAA and.906 save percentage.\textbackslash nMar. 10 6:15 PM PT7:15[...continues] &   stopped 30 of 34 shots in the Blue Jackets' 6-2 win over the Canadiens on Monday. The victory pushed him to 31-21-1 on the season with a 2.72 GAA and.908 save percentage.\textbackslash nMar. 10 11:49 AM PT12:49 PM [...continues] &     0.62 &     \hlred{2.40} &       1.33 &      1.45 \\\midrule
...cond season at Hall Bros Oval.\textbackslash nThe defender also admitted his surprise at Young’s run to the finals but credited the injection of youth into the side.\textbackslash n“We were really in a building phase last year and &  we copped a few floggings with all those juniors blokes coming in,” Galvin said.\textbackslash n“Now, we’ve kept that core group together for two years and I think we’ve come along quicker than we anticipated.\textbackslash nROCK[...continues] &  we copped a few floggings with all those juniors blokes coming in,” Galvin said.\textbackslash n“Now, we’ve kept that core group together for two years and I think we’ve come along quicker than we anticipated.\textbackslash n“Tha[...continues] &  we copped a few floggings with all those juniors blokes coming in,” Galvin said.\textbackslash n“Now, we’ve kept that core group together for two years and I think we’ve come along quicker than we anticipated.\textbackslash n“Tha[...continues] &     0.58 & \hlred{-1.13} &       1.05 &      1.04 \\
\bottomrule
\end{tabular}
\caption{Selected outputs from non-watermarked (NW) and watermarked (W) multinomial sampling using $\gamma=0.5$ and $\delta=2.0$. The examples in the first two rows have high entropy and correspondingly high $z$-scores, without any perceptible degradation in output quality. \textit{The two lower rows are failure cases where the watermark is too weak to be detected} -- they have low entropy and corresponding low $z$-scores. 
Anecdotally, failure cases typically seem to involve  data memorization in which the model regurgitates a near-copy of human text.  Note the output similarity between the generated and ``real'' human text in the bottom two rows. Memorization leads to large, high confidence logit values that constrain the outputs. 
Another common factor in failure cases is templated outputs (see the date/time formatting in row 3) that constrain model choices. 
}
\label{tab:demo-examples}
\vspace{-.2cm}
\end{table*}

\subsection{Impact on quality of generated text} 
A soft watermark has very little impact on the perplexity of tokens with extremely high or low entropy. When the distribution produced by the language model is uniform (maximal entropy), the randomness of the green list results in tokens being uniformly sampled, and the perplexity remains untouched.  Conversely, in the case of minimal entropy, where all probability mass is concentrated on a single token, the soft watermark rule has no effect and there is once again no impact on perplexity.

The watermark rule does impact perplexity for tokens of moderate entropy.  In this case, we can provide the following simple bound that holds uniformly over all entropy values. 

\begin{theorem} \label{perpbound}
Consider a sequence $s^{(i)}, -N_p<i<T.$ Suppose the (non-watermarked) language model produces a probability vector $p^{(T)}$ for the token at position $T.$  The watermarked model predicts the token at position $T$ using modified probability vector $\hat p^{(T)}.$  The expected perplexity of the $T$th token with respect to the randomness of the red list partition is 
 $$\expect_{G, R} \sum_k \hat p^{(T)}_k \ln(p^{(T)}_k) \le (1+(\alpha-1)\gamma)P^*,$$
 \looseness -1 where $P^*=\sum_k p^{(T)}_k \ln(p^{(T)}_k)$ is the perplexity of the original model.
\end{theorem}

\section{Private Watermarking} 
\label{cryptosec}
\looseness -1  The watermark algorithms above are designed to be \textit{public}.
A watermark can also be operated in \textit{private mode}, in which the algorithm uses a random key that 
is kept secret and hosted behind a secure API. If the attacker has no knowledge of the key used to produce the red list, it becomes more difficult for the attacker to remove the watermark as the attacker does not know which tokens are in the red list. However, testing for the presence of the watermark now requires using the same secure API and, if this API is public, access needs to be monitored to prevent an adversary from making too many queries using minor variants of the same sequence. 

\sloppy
Let $F$ be a pseudorandom function (PRF) that, for simplicity, we view as accepting arbitrary length inputs and producing output as long as needed. $F$ could be a standard block cipher like AES or a cryptographic hash function like~SHA3. To create a private watermark, we first choose a random key $\key$; 
a private red list for token $s^{(t)}$ can then be generated in a manner similar to what was described earlier, but now by first computing $F_{\key}(s^{(t-h)},\cdots, s^{(t-1)})$, a pseudorandom function evaluated on the prior $h$ tokens. 

An attacker can discover the watermarking rules by observing occurrences of token tuples in generated text and tabulating the frequencies of the immediately subsequent tokens, even if the underlying key is unknown. To tabulate every red list in such a brute-force attack, $|\vocab|^{1 + h}$ tokens need to be submitted to the detection API.
When $h=1$, the red lists produced by many tokens could be discovered (at least partially) with conceivable effort. 
This brute-force method is ineffective for $h \gg 1$, as there is now a unique red list for each ordered combination of words.  
At the same time, large values of $h$ decrease watermark robustness when a naive method is used.  When, say, $h=5$ consecutive tokens are used to produce a red list, an adversarial change to just one of those tokens randomizes the red list for 5 different downstream tokens, increasing the number of red list words by $2.5$ (in expectation) if $\gamma=.5$. We call this downstream impact {\em attack amplification}. To limit amplification, we suggest using a small window ($h=2$ or~$3$) when using the naive watermarking rule. 

\begin{algorithm}[t]
   \caption{Robust Private Watermarking}
   \label{alg:private_robust_watermark}
\begin{algorithmic}
\STATE \textbf{Input:} prompt $s^{(-N_p)}\cdots s^{(-1)}$

\hspace{10.5mm}PRF $F$ with key $\key$

\hspace{10.5mm}hardness parameter $\delta>0$

\hspace{10.5mm}window width $h>0$

\FOR{$t=0,1,\cdots$}
 \STATE 
 \begin{enumerate}
\item Apply the language model to $s^{(-N_p)}\cdots s^{(t-1)}$ to get a logit vector $l^{(t)}$ over the vocabulary. 
\item  Sort the vocabulary so $l^{(t)}$ is in descending order. Set $k = 0$, the index of the most likely token.
\item Temporarily set $s^{(t)}$ to be the $k$th token in the vocabulary. Compute
$$H_i = F_{\key}(s^{(t)},s^{(t-i)}) \text{ for } 1\le i \le h.$$
\item Set $i^\star = \arg\min_{i>0} H_i$.
\item Using $H_{i^\star}$ as a seed, produce a random bit to decide if token $k$ is on the green or red list.
\end{enumerate}
\IF{ green list is chosen}
\STATE keep $s^{(t)}$ and continue.
\ELSIF{ red list is chosen, and $l^{(t)}_{k+1}< l^{(t)}_0 - \delta,$} 
\STATE choose  $s^{(t)}$ to be the most likely ($k=0$) token, which is in the red list, and continue.
\ELSE 
\STATE set $k\gets k+1,$ \textbf{goto} to step 3.
\ENDIF
\ENDFOR
\end{algorithmic}
\end{algorithm}

\vspace{.5cm}
\looseness -1 When a wider window $h$ is desired, more complex, \textit{robust watermarking rules} can achieve security against brute-force attacks without attack amplification. We describe such a rule in \cref{alg:private_robust_watermark}.  Here, the red list for $s^{(t)}$ depends on {\em itself}, and additionally on one prior token $s^{(t-i^\star)}$ chosen using a pseudo-random rule. To satisfy this self-hash condition, we iteratively test different tokens as $s^{(t)},$ from highest logit to least logit, until the red list rule is satisfied.  If, during this search, the logit of the test token falls by more than $\delta$, we give up and accept the token in the red list with largest logit.

\cref{alg:private_robust_watermark} has several nice security properties. When one of the prior $h$ tokens is changed, the watermark at position $t$ changes with probability only~$1/h$. As such, this rule is free of attack amplification; in expectation, a change to a token results in one additional red list token.\footnote{When $\gamma=.5$, the flipped token is in the red list $1/2$ of the time, and one of the $h$ downstream red lists is expected to randomize, resulting in another $1/2$ red list token.} 
Like the naive method with $h=2$, there are $|\vocab|^2$ unique red lists, but now the choice of the index $i^\star$ depends on combinations of $s^{(t)}$ and all $h$ tokens before it, which hides the choice of tokens used as input to~$F$.  For simplicity, \cref{alg:private_robust_watermark} is presented as a greedy sampler, but can be easily extended to handle multinomial sampling or beam search. 

\textbf{Boosting Watermark Privacy with Multiple Keys.}
A straightforward add-on to significantly boost the difficulty of brute-forcing a hidden watermark scheme is to have $k>1$ different hidden keys, and to randomly choose one for each generation.  At detection time, we run $k$ tests, one for each of the possible keys. This comes at the cost of only a minor decrease in power, as we need to correct for multiple hypotheses, for example via Bonferroni correction. 
When $\gamma=.5$ (half the vocabulary is colored green), lists should further be constructed so that each word is green for $k/2$ of the keys and red for the other $k/2.$  In this way, the lists can no longer be discovered by brute forced frequency analysis, as there is no observable bias when averaging over a large number of separately generated strings.   


\vspace{-.2cm}
\section{Experiments} \label{sec:experiments}

\begin{figure*}[h]

\begin{center}
\includegraphics[width=.95\columnwidth]{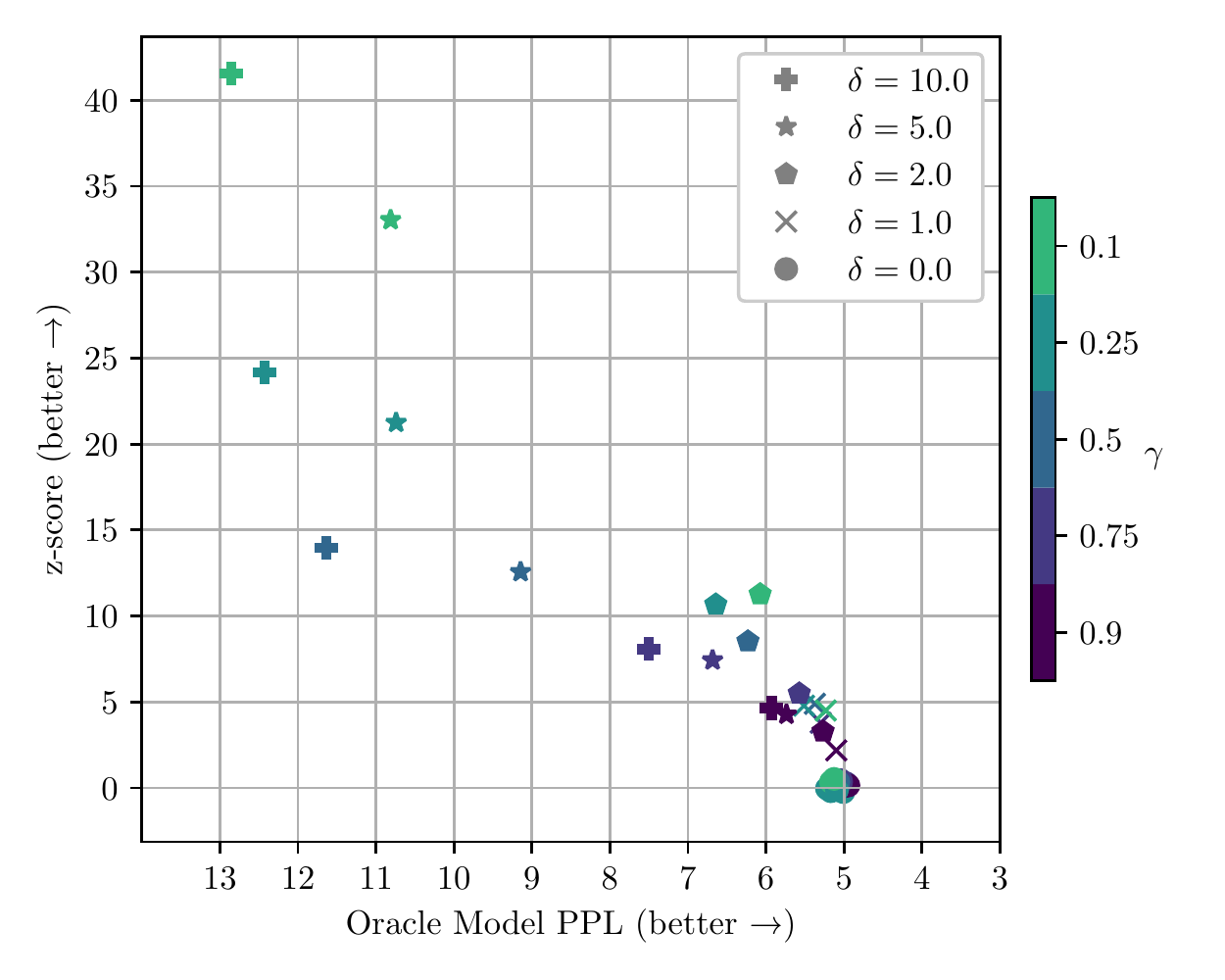}
\includegraphics[width=.95\columnwidth]{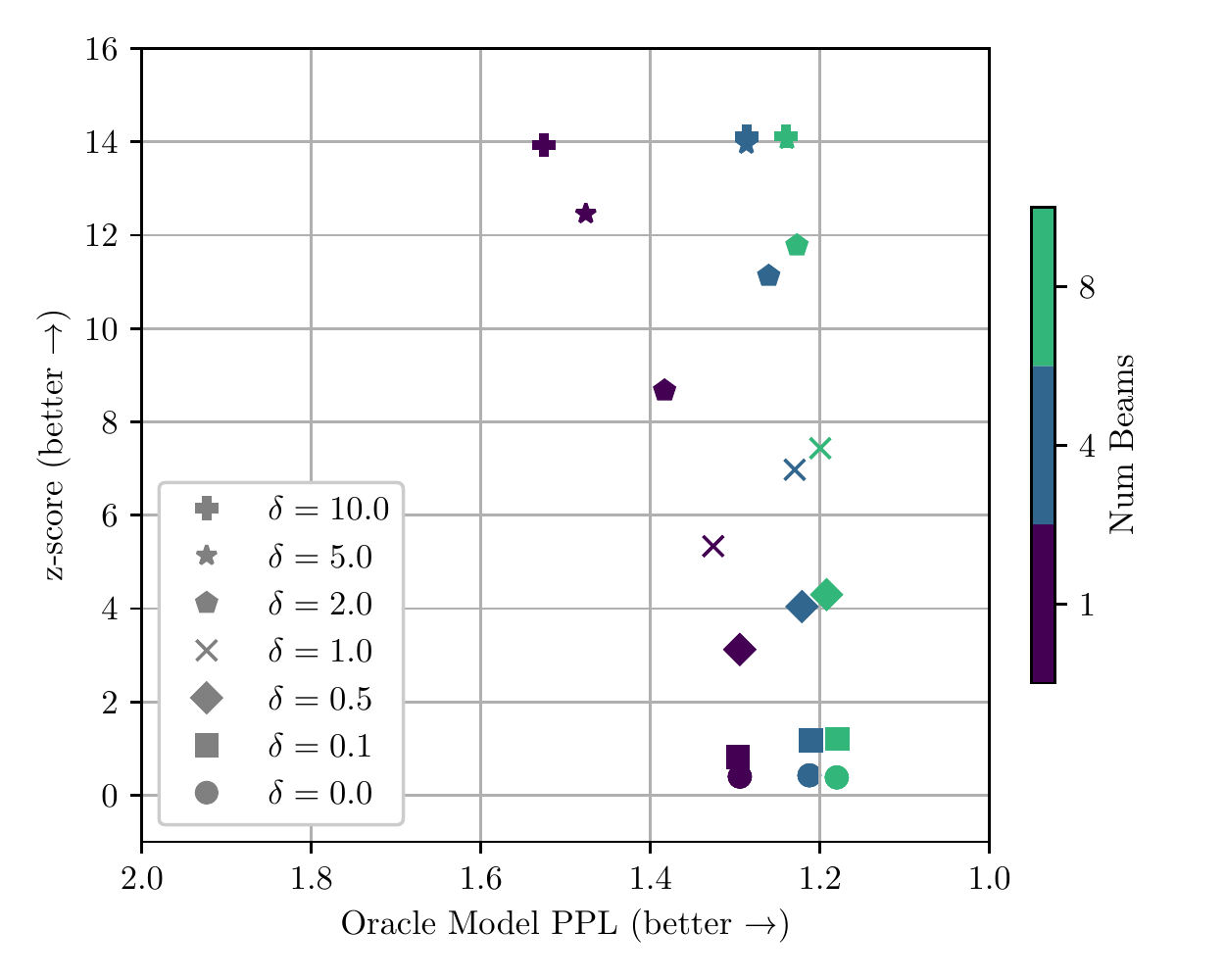}
\vspace{-4mm}
\caption{The tradeoff between average z-score and language model perplexity for $T=200\pm5$ tokens. 
 (left) Multinomial sampling. (right)  Greedy and beam search with 4 and 8 beams for $\gamma=.5$. 
Beam search promotes higher green list usage and thus larger $z$-scores with smaller impact to model quality (perplexity, PPL). 
}
\label{fig:pareto}
\end{center}
\vskip -0.2in
\end{figure*}

\begin{figure*}[h]
    \centering
    \subfigure[]{\includegraphics[width=0.32\textwidth,trim={0.5cm 0.4cm 0.1cm 0},clip]{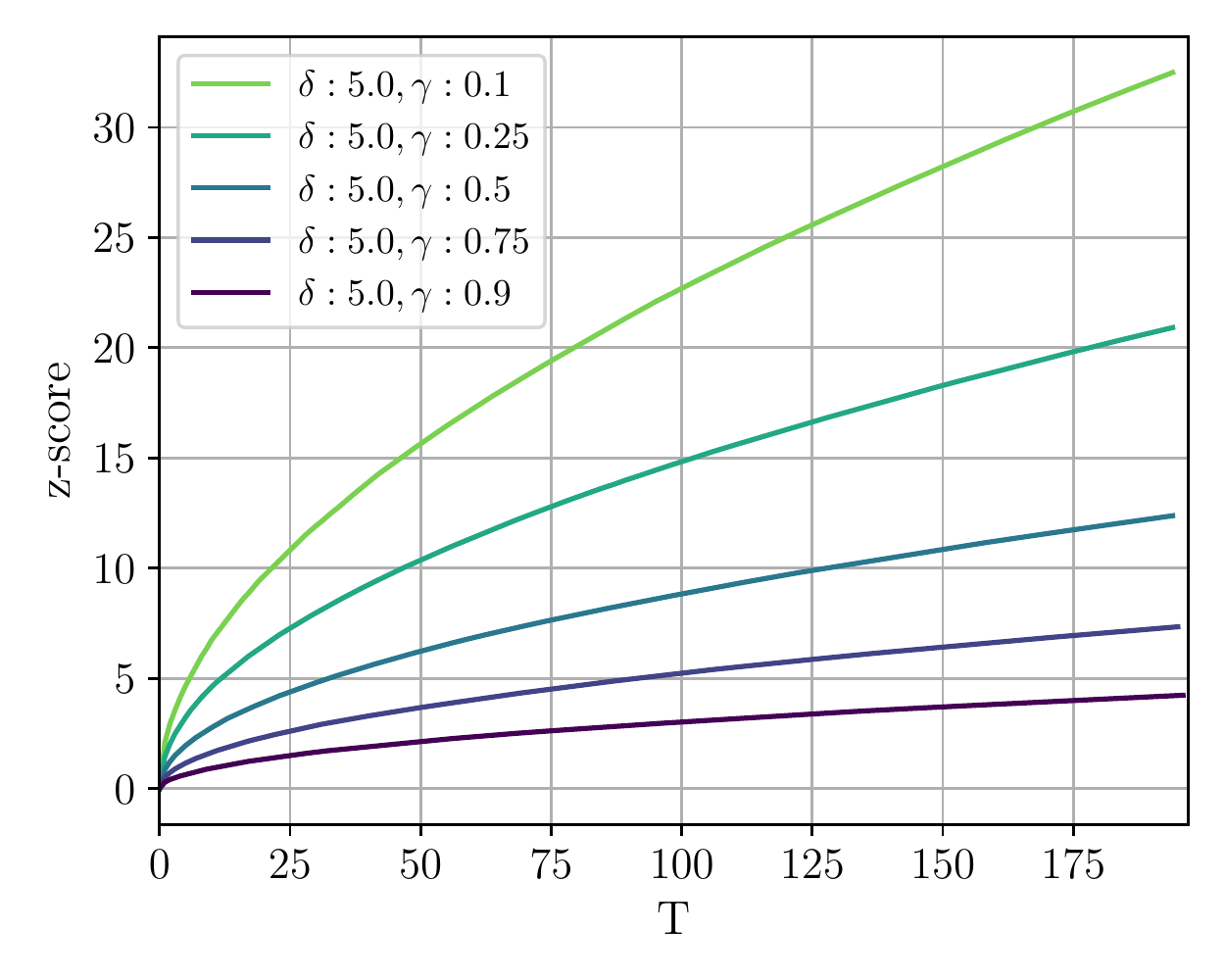}\label{fig:z-vs-t-ablate-gamma}}
    \subfigure[]{\includegraphics[width=0.32\textwidth,trim={0.5cm 0.4cm 0.1cm 0},clip]{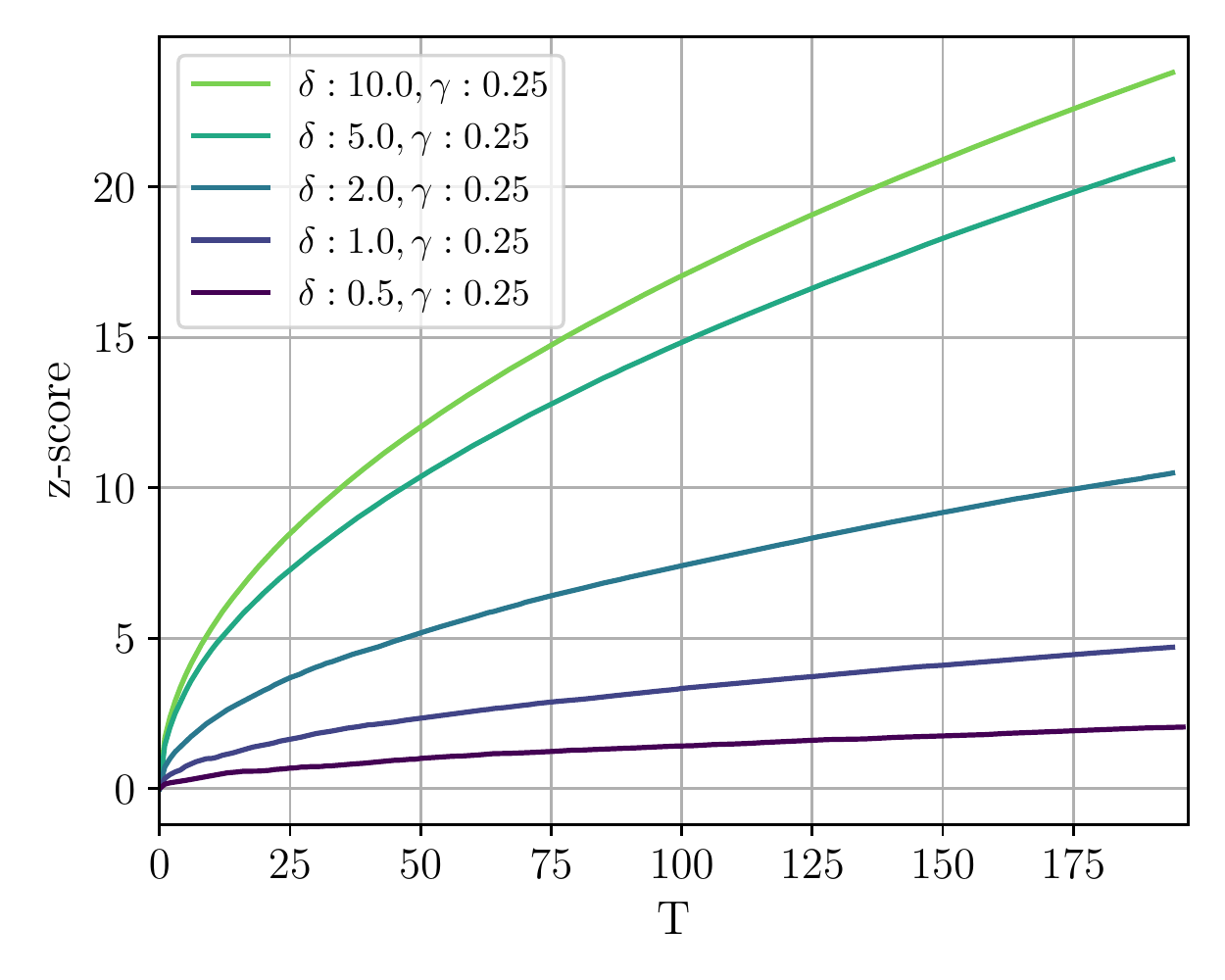}\label{fig:z-vs-t-ablate-delta}}
    \subfigure[]{\includegraphics[width=0.32\textwidth,trim={0.5cm 0.4cm 0.1cm 0},clip]{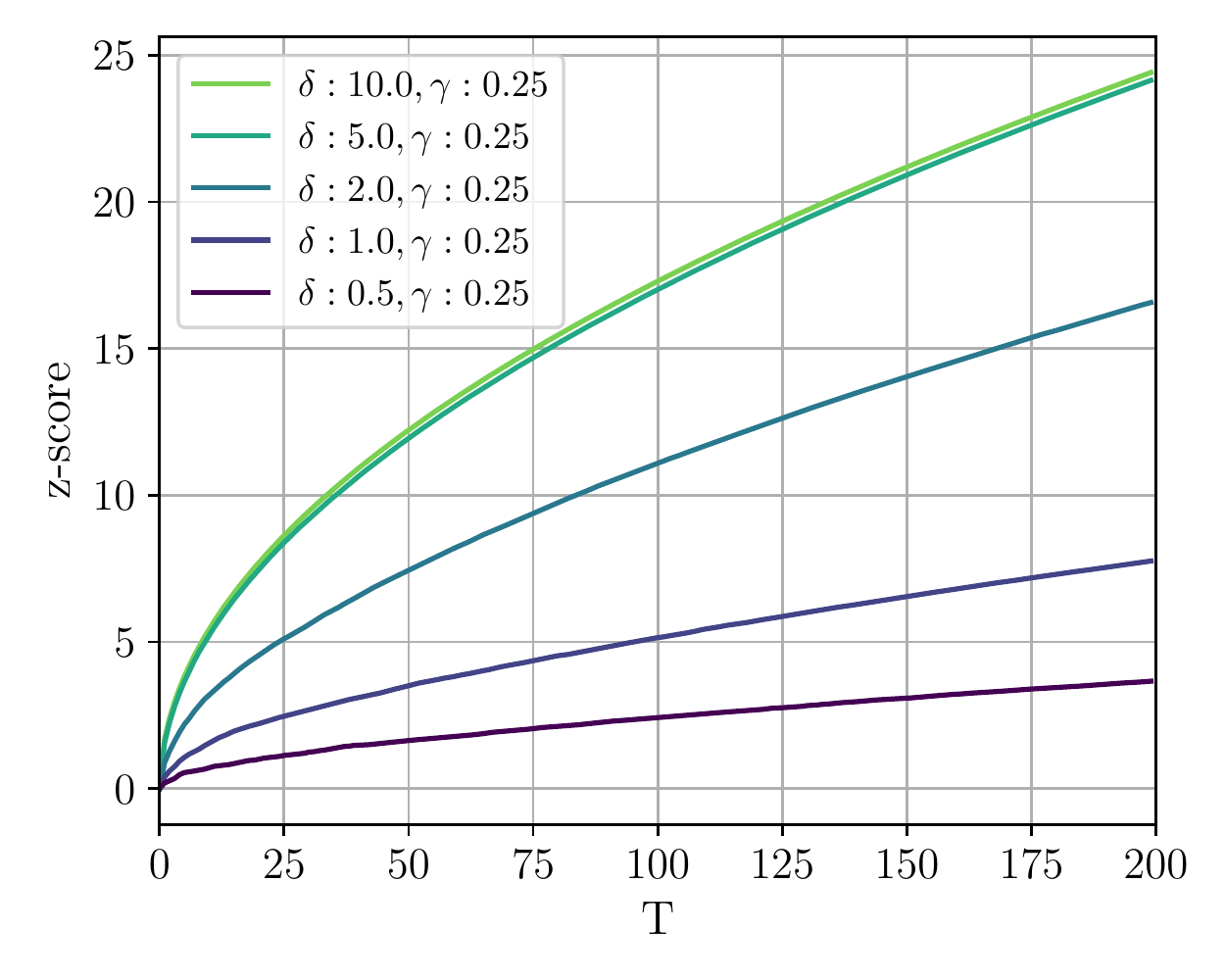}\label{fig:z-vs-t-beams-ablate-delta}}
    \vspace{-3mm}
    \caption{The average $z$-score as a function of $T$ the token length of the generated text.  (a) The dependence of the $z$-score on the green list size parameter $\gamma$, under multinomial sampling. (b) The effect of $\delta$ on $z$-score, under multinomial sampling. (c) The impact of the green list size parameter $\gamma$ on the $z$-score, but with greedy decoding using 8-way beam search. }
    \label{fig:z-vs-t}
    \vspace{-3mm}
\end{figure*}

In this section we explore the behavior of the watermark using the OPT-1.3B model~\citep{zhang2022opt}. We measure watermark strength using the rate of type-I errors (human text falsely flagged as watermarked) and type-II errors (watermarked text not detected).  

We implement the proposed watermark using the Pytorch backend of the Huggingface library \citep{wolf_huggingfaces_2020}. The \texttt{generate} API provides useful abstractions, including modules for \textit{warping} the logit distribution that comes out of the language model. 
We generate red lists using the torch random number generator and one previous token as described in \cref{sec:soft-watermark}.

\textbf{Datasets and Prompts.}
To simulate a variety of realistic language modeling scenarios we slice and dice a random selection of texts from the news-like subset of  the C4 dataset \citep{2019t5}. 
For each random string, we trim a fixed length of tokens from the end and treat them as a ``baseline'' completion.  The remaining tokens are a prompt. For the experimental runs using multinomial sampling, we pull examples from the dataset until we achieve at least 500 of generations with length $T=200\pm5$ tokens. In the runs using greedy and beam search decoding, we suppress the EOS token during generation  to combat the tendency of beam search to generate short sequences. We then truncate all sequences to $T=200$.
A larger \textit{oracle} language model (OPT-2.7B) is used to compute perplexity (PPL) for the generated completions and for the human baseline. 

\textbf{Watermark Strength vs Text Quality.}
One can achieve a very strong watermark for short sequences by choosing a small green list size $\gamma$ and a large green list bias $\delta.$  However, creating a stronger watermark may distort generated text.  \cref{fig:pareto} (left) shows the tradeoff between watermark strength ($z$-score) and text quality (perplexity) for various combinations of watermarking parameters. We compute results using $500\pm10$ sequences of length $T=200\pm5$ tokens for each parameter choice.  Interestingly, we see that a small green list, $\gamma=.1$ is pareto-optimal. 

In addition to these quantitative results, we show examples of real prompts and watermarked outputs in \cref{tab:demo-examples} to provide a qualitative sense for the behavior of the test statistic and quality measurement on different kinds of prompts. Additional examples are compiled in \cref{sec:sample-outputs}.

\begin{figure*}
    \centering
    \subfigure[]{\includegraphics[width=0.24\textwidth,trim={0.42cm 0.1cm 0.9cm 0.5cm},clip]{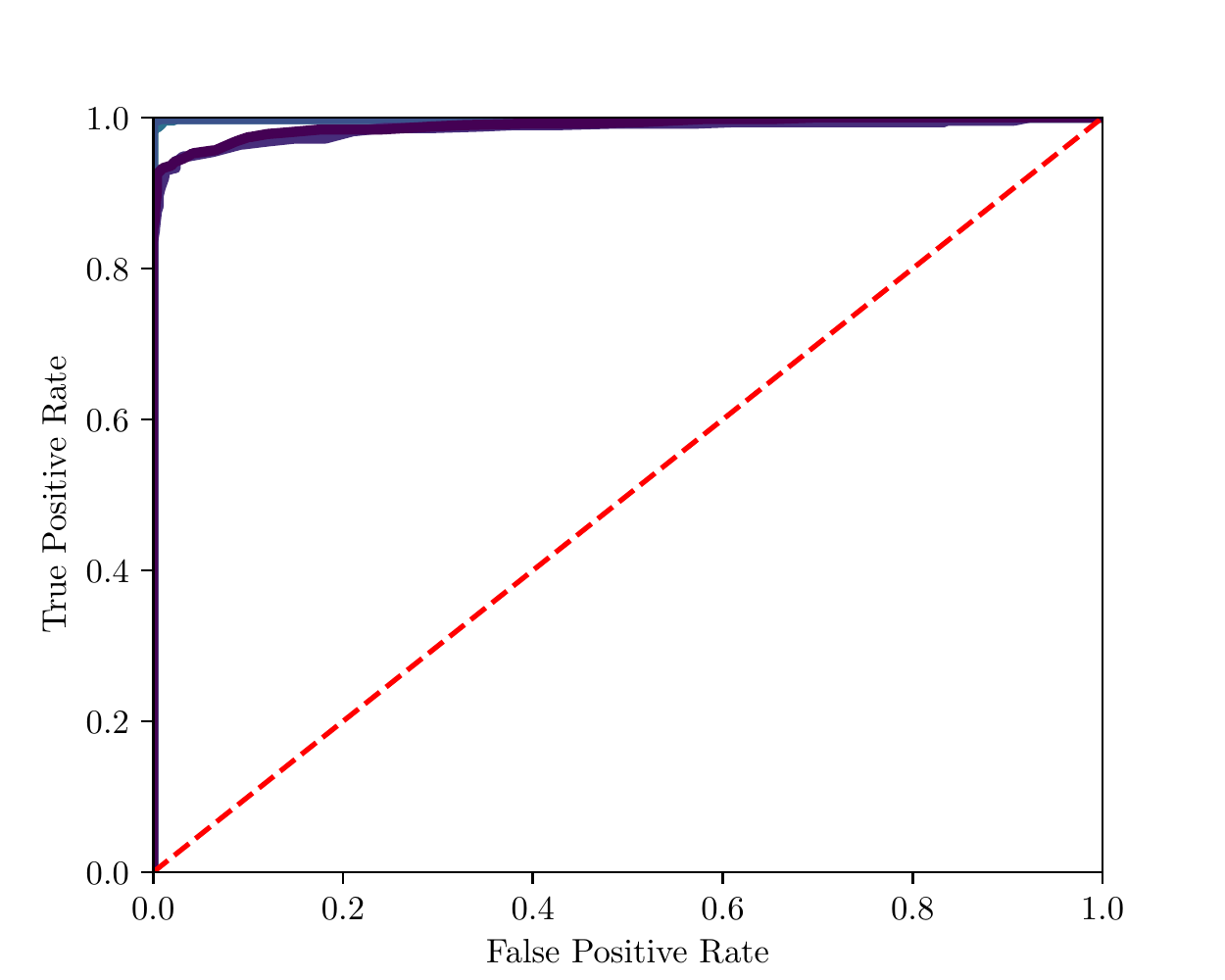}}
    \subfigure[]{\includegraphics[width=0.24\textwidth,trim={0.42cm 0.1cm 0.9cm 0.5cm},clip]{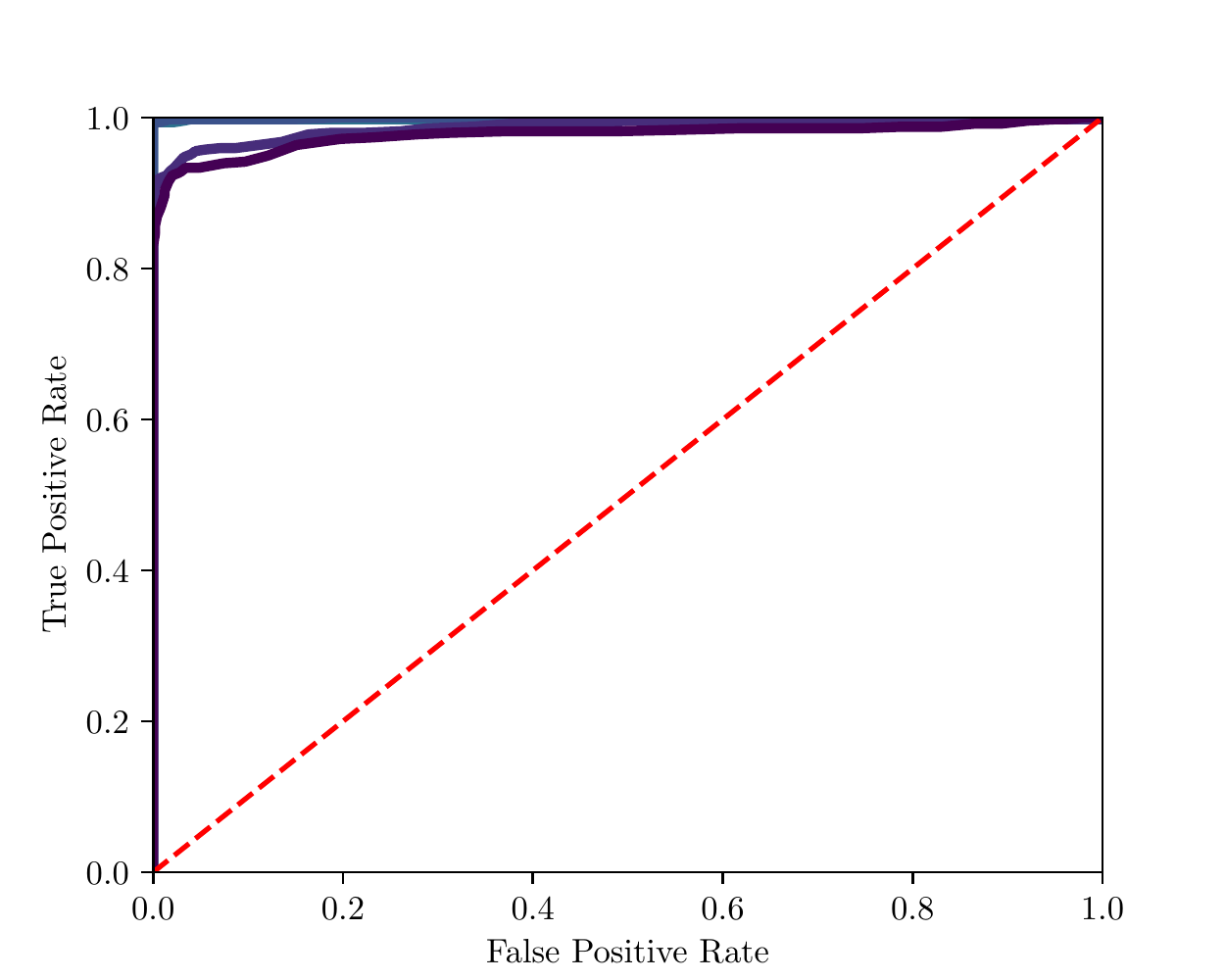}
    }
    \subfigure[]{\includegraphics[width=0.24\textwidth,trim={0.42cm 0.1cm 0.9cm 0.9cm},clip]{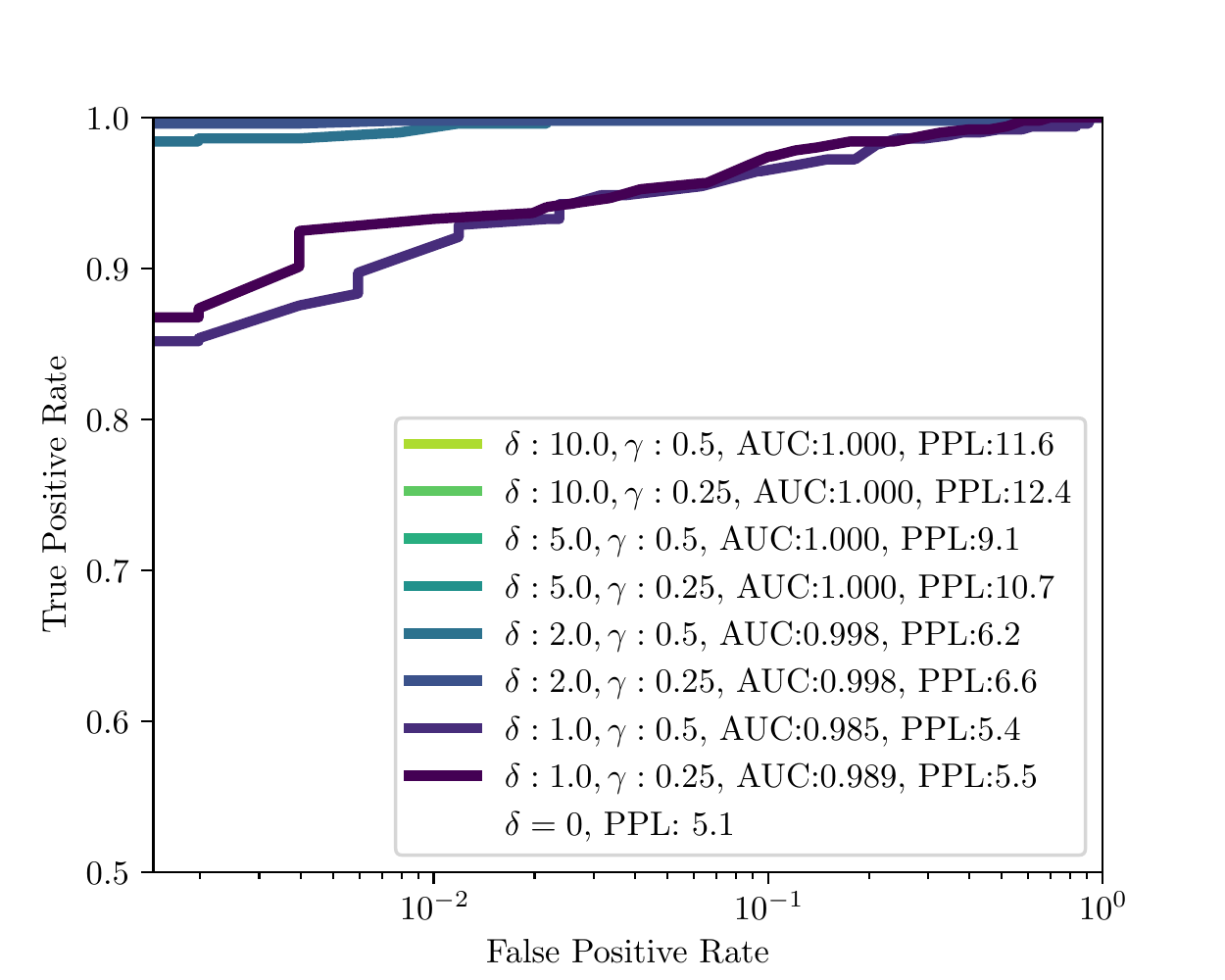}\label{fig:roc-auc-zoom}}
    \subfigure[]{\includegraphics[width=0.24\textwidth,trim={0.42cm 0.1cm 0.9cm 0.9cm},clip]{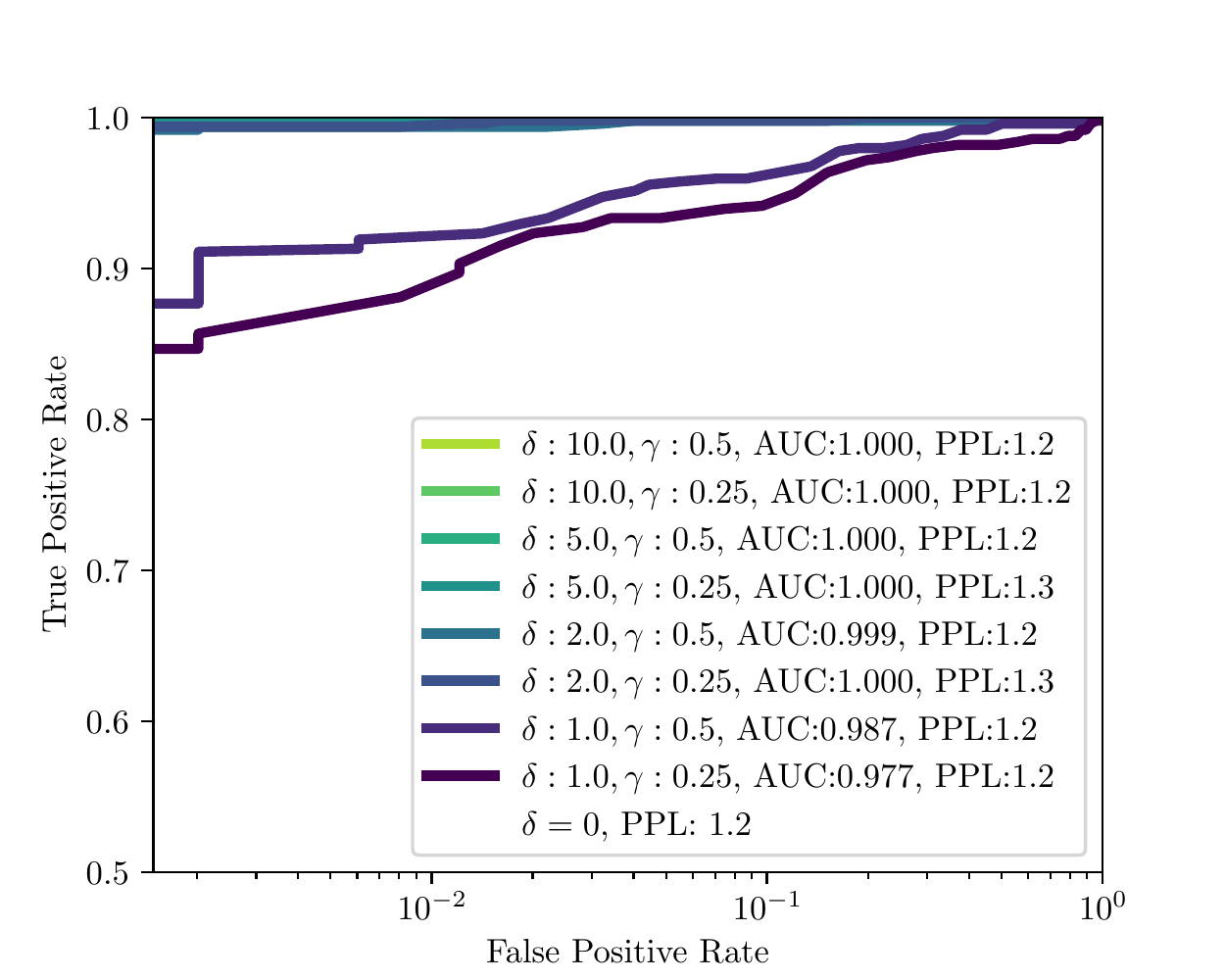}\label{fig:roc-auc-beams-zoom}}
    \vspace{-3mm}
    \caption{\looseness -1 ROC curves with AUC values for watermark detection. Several choices of watermark parameter $\delta$ are shown for \textbf{(a)} multi-nomial sampling  and \textbf{(b)} greedy decoding with 8-way beam search. \textbf{(c,d)} The same charts with semilog axes. Higher $\delta$ values achieve stronger performance, but additionally we see that for a given $\delta$, the beam search allows the watermark to capture slightly more AUC than the corresponding parameters under the multinomial sampling scheme.}
    \label{fig:roc-auc-beams-compressed}
\end{figure*}

\begin{table*}[h]
\small
\centering
\begin{tabular}{lrrrrrrrrrrr}
\toprule
&&&&\multicolumn{4}{c}{z=4.0} & \multicolumn{4}{c}{z=5.0}\\
\midrule
sampling & $\delta$ & $\gamma$ & count &  FPR &  TNR &  TPR &  FNR &  FPR &  TNR &  TPR &  FNR \\\midrule
    m-nom. &    1.0 &   0.50 &    506 &                  0.0 &                  1.0 &            0.767 &            0.233 &                  0.0 &                  1.0 &            0.504 &            0.496 \\
         m-nom. &    1.0 &   0.25 &    506 &                  0.0 &                  1.0 &            0.729 &            0.271 &                  0.0 &                  1.0 &            0.482 &            0.518 \\
         m-nom. &    2.0 &   0.50 &    507 &                  0.0 &                  1.0 &            0.984 &            0.016 &                  0.0 &                  1.0 &            0.978 &            0.022 \\
         m-nom. &    2.0 &   0.25 &    505 &                  0.0 &                  1.0 &            0.994 &            0.006 &                  0.0 &                  1.0 &            0.988 &            0.012 \\
         m-nom. &    5.0 &   0.50 &    504 &                  0.0 &                  1.0 &            0.996 &            0.004 &                  0.0 &                  1.0 &            0.992 &            0.008 \\
         m-nom. &    5.0 &   0.25 &    503 &                  0.0 &                  1.0 &            1.000 &            0.000 &                  0.0 &                  1.0 &            0.998 &            0.002 \\
         8-beams &    1.0 &   0.50 &    495 &                  0.0 &                  1.0 &            0.873 &            0.127 &                  0.0 &                  1.0 &            0.812 &            0.188 \\
        8-beams &    1.0 &   0.25 &    496 &                  0.0 &                  1.0 &            0.819 &            0.181 &                  0.0 &                  1.0 &            0.770 &            0.230 \\
        8-beams &    2.0 &   0.50 &    496 &                  0.0 &                  1.0 &            0.992 &            0.008 &                  0.0 &                  1.0 &            0.984 &            0.016 \\
        8-beams &    2.0 &   0.25 &    496 &                  0.0 &                  1.0 &            0.994 &            0.006 &                  0.0 &                  1.0 &            0.990 &            0.010 \\
        8-beams &    5.0 &   0.50 &    496 &                  0.0 &                  1.0 &            1.000 &            0.000 &                  0.0 &                  1.0 &            1.000 &            0.000 \\
        8-beams &    5.0 &   0.25 &    496 &                  0.0 &                  1.0 &            1.000 &            0.000 &                  0.0 &                  1.0 &            1.000 &            0.000 \\
\bottomrule
\end{tabular}
\caption{Empirical error rates for watermark detection using multinomial sampling and beam search. Each row is averaged over $\sim500$ generated sequences of length $T=200\pm5$. 
A maximum of one type-I (false positive) error was observed for any given run. All soft watermarks at $\delta=2.0$ incur at most $1.6\%$ ($8/500$) type-II error at $z=4$. No type-II errors occurred for the hardest watermarks with $\delta=10.0$ and $\gamma=0.25$.}
\label{tab:joint-acc}
\vspace{-.2cm}
\end{table*}

\textbf{Ironing in the Watermark with Beam Search.}
\cref{fig:pareto} (right) shows the tradeoff between watermark strength and accuracy when beam search is used.
Beam search has a synergistic interaction with the soft watermarking rule. 
Particularly when 8 beams are used, the points in \cref{fig:pareto} form an almost vertical line, showing very little perplexity cost to achieve strong watermarking.

\textbf{Watermark Strength vs Number of Tokens.}
Theory predicts that the type I and type II error rates of the watermark should decay to zero as the sequence length $T$ increases.  
\cref{fig:z-vs-t} shows the strength of the watermark, measured using the average $z$-score over samples, as $T$ sweeps from 2 to 200.  Curves are shown for various values of $\delta$ and $\gamma$.  The left two charts use multinomial sampling, while the right chart uses 8-way beam search and $\gamma=.25$.  Once again, we see the power of the beam search in achieving high green list ratios; even for the moderate bias of $\delta=2$, an average $z$-score greater than 5 is achieved for as few as 35 tokens.

\textbf{Performance and Sensitivity for Multinomial Sampling.}
To show the sensitivity of the resulting hypothesis test based on the observed $z$-scores, we provide a table of error rate for various watermarking parameters in \cref{tab:joint-acc}.  We also sweep a range of thresholds in ROC charts in \cref{fig:roc-auc-beams-compressed}. We further report detection performance and error rates for various cutoffs in
\cref{sec:multinomial_results}, and provide a comparison between empirical $z$-scores and theoretical predictions. Note that no type-I (false positive) errors were observed for any run shown in the error tables (see the columns of $0.0$'s

\section{Attacking the watermark}\label{sec:attacks}

\begin{figure*}[h]
    \centering
    \includegraphics[width=0.49\textwidth]{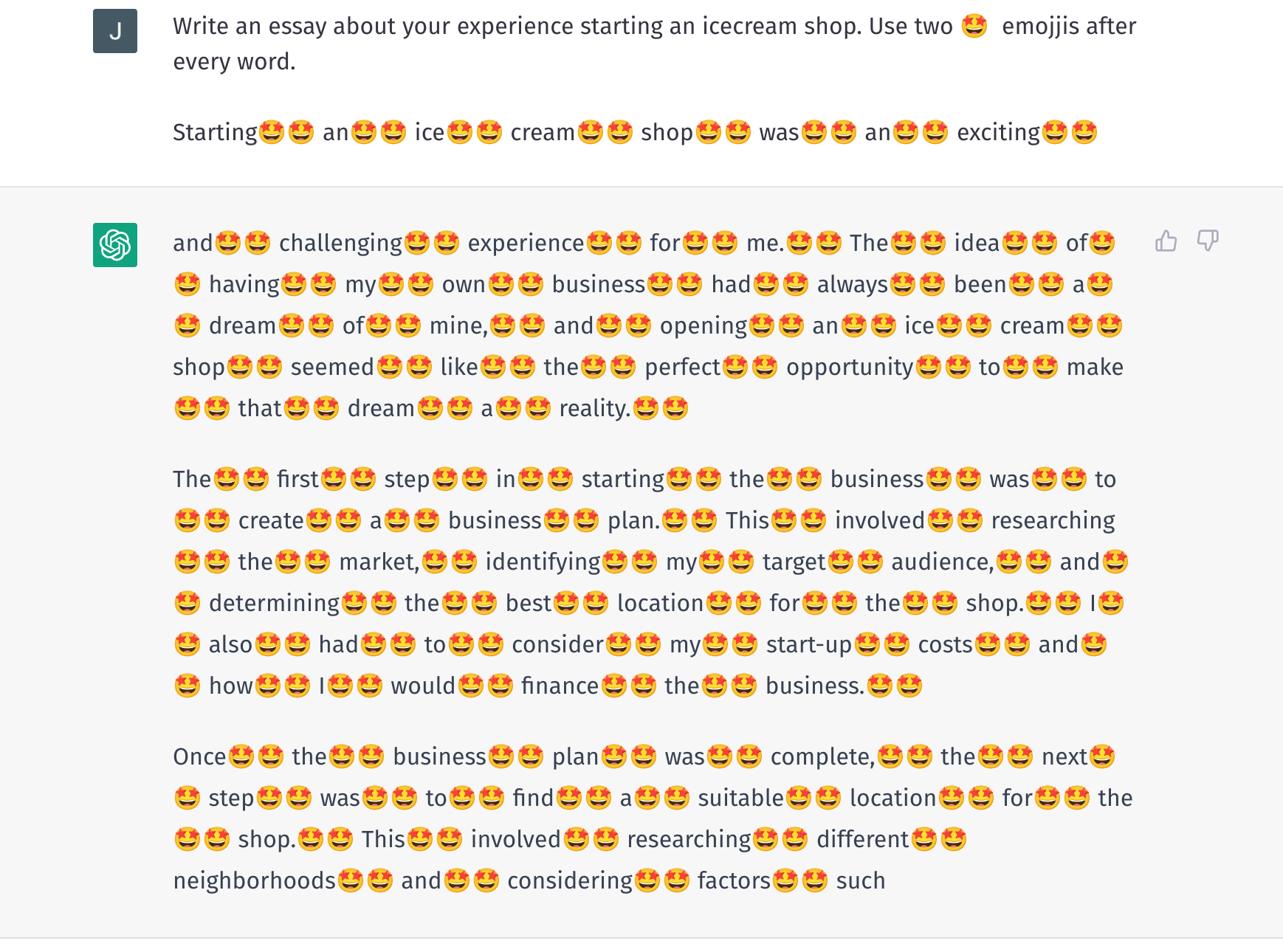}
    \includegraphics[trim={0 4cm 0 0},clip, width=0.49\textwidth]{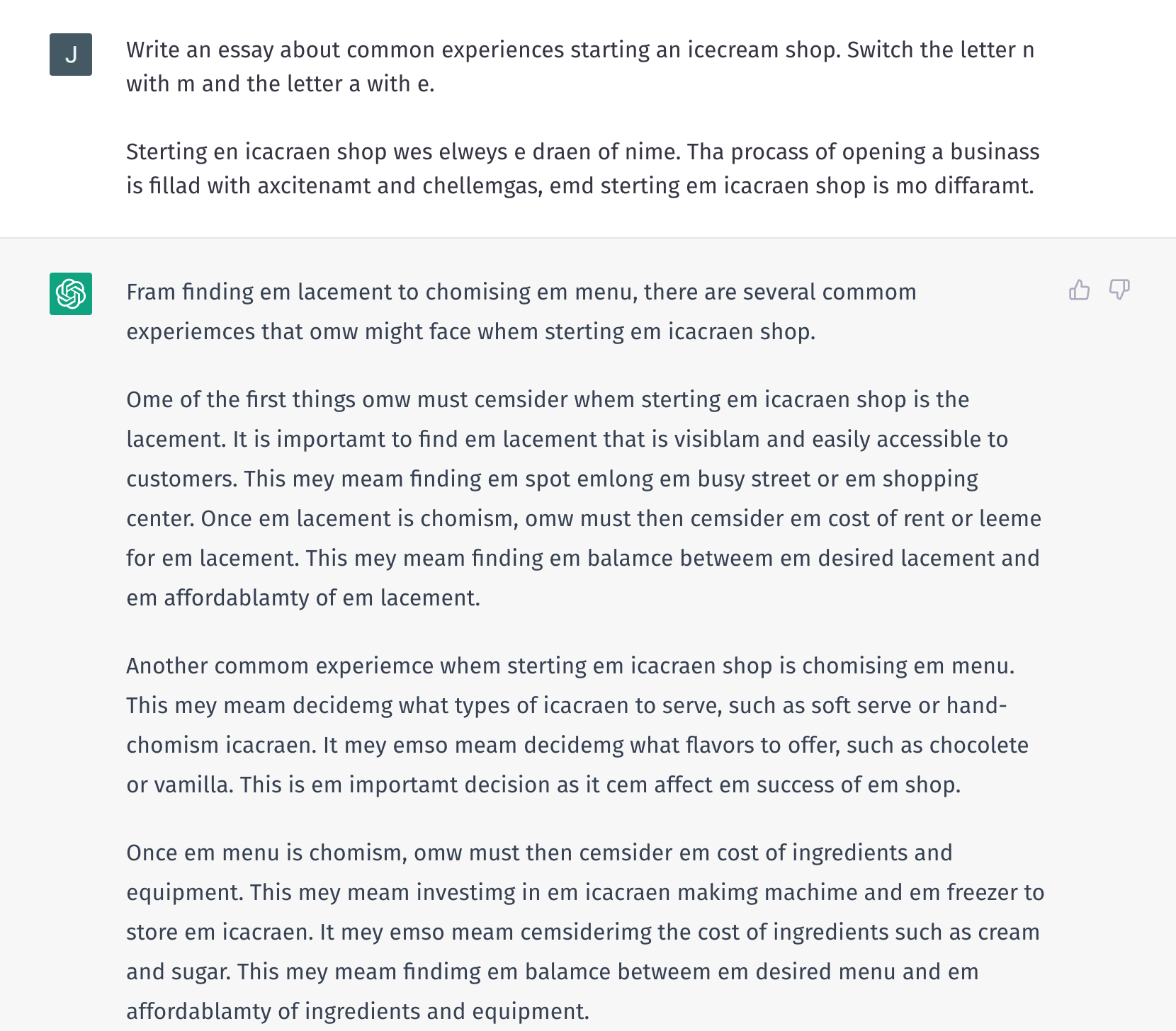}
    \caption{\textbf{Left:} \looseness -1 The ``Emoji Attack'' of \citet{goodside_there_2023} shown on the chatGPT web API on Dec15th 2022. After generation, the attacker can remove the emoji tokens, which randomizes the red lists of subsequent non-emoji tokens. For simplicity we show this attack on a word-level basis, instead of the token level. \textbf{Right:} A more complicated character substitution attack, also against chatGPT.  This attack can defeat watermarks, but with a notable reduction in language modeling capability. }
    \label{fig:emojji_attack}
    \vspace{-.25cm}
\end{figure*}

Like any software tool, care must be taken when implementing a watermark and watermark detector so that security is maintained. Otherwise, an adversarial user may modify text to add red list tokens, and thus avoid detection. In many cases, simple attacks can be avoided by properly normalizing text before hashes are computed. We now discuss a range of attacks that we are currently aware of, and methods to mitigate them.  We assume a threat model in which an attacker must create watermark-free text using a combination of a private watermarked model and other public models, but the public models are much weaker than the watermarked model.  We only consider attacks that maintain text of quality similar to the raw private model. 

Three types of attacks are possible.
\textbf{Text insertion} attacks add additional tokens after generation that may be in the red list and may alter the red list computation of downstream tokens.
\textbf{Text deletion} removes tokens from the generated text, potentially removing tokens in the green list and modifying downstream red lists. 
This attack increases the monetary costs of generation, as the attacker is ``wasting'' tokens, and may reduce text quality due to effectively decreased LM context width.
\textbf{Text substitution} swaps one token with another, potentially introducing one red list token, and possibly causing downstream red listing. This attack can be automated through dictionary or LM substitution, but may reduce the quality of the generated text. 

Below we catalog a range of attacks that fall into these categories.

\textbf{Paraphrasing Attacks}.
A baseline substitution attack is manual paraphrasing by the human attacker. This attack is technically outside the threat model we are interested in, as it requires extensive human intervention. 
Note that, especially on longer text fragments such as essays, a few sentences that are partially or not at all paraphrased can be sufficient to trigger watermark detection at a statistically significant threshold.

A more scalable version of this attack is to use automated paraphrasing. An attacker that has access to a public language model can use this model to rephrase the output of the generated model. We provide an experimental evaluation of this attack in \cref{sec:replacement-attack}. \emph{Here, it is crucial to note the trade-off that an attacker is making: The attacker is using a weaker paraphrasing model to modify the text, reducing both watermark strength and text fluency.} If the attacker had an equally strong language model at hand, there would be no need to use the watermarked API, the attacker could generate their own text.

\textbf{Discreet Alterations.}
An attacker could make small alterations, adding additional whitespaces, or misspelling a few words to impact the computation of the hash.
A well-constructed watermark should normalize text to ignore explicit whitespaces when computing the hash. Changing the spelling of many words is likely to severely degrade the quality of text. When implemented carefully, surface level alterations should not pose a serious threat to a watermark. 

\textbf{Tokenization Attacks.}. 
An attacker can modify text so that sub-word tokenization of a subsequent word changes. For example (again with BPE), if the text fragment \texttt{ life.\textbackslash nVerrilius} is modified to \texttt{ life. Verrilius} (i.e. \texttt{\textbackslash n} is replaced), then the tokenization of the succeeding word also switches from \texttt{V\_err\_ili\_us} to \texttt{Ver\_r\_ili\_us}.  This results in more red list tokens than one would expect from a single insertion. The attack can contribute to the effectiveness of a more powerful attack, but most tokens in a default sentence will not be vulnerable. 

\textbf{Homoglyph and Zero-Width Attacks.}
\looseness -1  This is a special case of the discreet alteration attack.  The effect of tokenization attacks can be multiplied through homoglyph attacks \citep{gabrilovich_homograph_2002}. Homoglyphs attacks are based on the fact that unicode characters are not unique, with multiple unicode IDs resolving to the same (or a very similar-looking) letter. This breaks tokenization, for example the word \texttt{Lighthouse} (two tokens) expands to 9 different tokens if \texttt{i} and \texttt{s} are replaced with their equivalent Cyrillic unicode characters.
Security against Homoglyph and tokenization attacks can be maintained using input normalization before the text is tested for watermarks, for example via canonicalization as in \citet{helfrich_dual_2012}. 
Otherwise, simple replacements of characters with their homoglyphs could break enough tokens to remove the watermark.

Likewise, there are zero-width joiner/non-joiner unicode characters that encode zero-width whitespace and hence are effectively invisible in most languages. Like homoglyphs, these characters must be removed through canonicalization \citep{pajola_fall_2021,boucher_bad_2022}.

\textbf{Generative Attacks.}
Generative attacks abuse the capability of large language models for in-context learning, and prompt the model to change its output in a predictable and easily reversible way. For example, the Emoji attack of \citet{goodside_there_2023} proceeds by prompting the model to generate an emoji after every token, see \cref{fig:emojji_attack}, left. These emojis can be removed, randomizing the red list for subsequent tokens. More broadly, all attacks that prompt the model to change its output ``language'' in a predictable way can potentially cause this, for example prompting the model to replace all letters \texttt{a} with \texttt{e}, see \cref{fig:emojji_attack}, right. Or, as a reverse homoglyph attack, prompting the model to ``switch the letter i with i", where the second i is a Cyrillic letter.

These attacks are the strongest tools against watermarking to our knowledge, but also require a strong LM with the capacity to follow the prompted rule without a loss in output quality. Additionally, this increases the cost of text generation by requiring more tokens than usual to be generated and reducing effective context width. 

\looseness -1 A defense against these attacks is to include negative examples of such prompts during finetuning, training the model to reject these requests.  Note that instruction finetuning is already common (for example in ChatGPT) for other categories of malicious prompts, using reinforcement learning protocols (RLHF) \citep{christiano_deep_2017,ouyang_training_2022,bai_constitutional_2022}. 

\begin{figure}[t]
    \centering
    \includegraphics[width=0.49\textwidth]{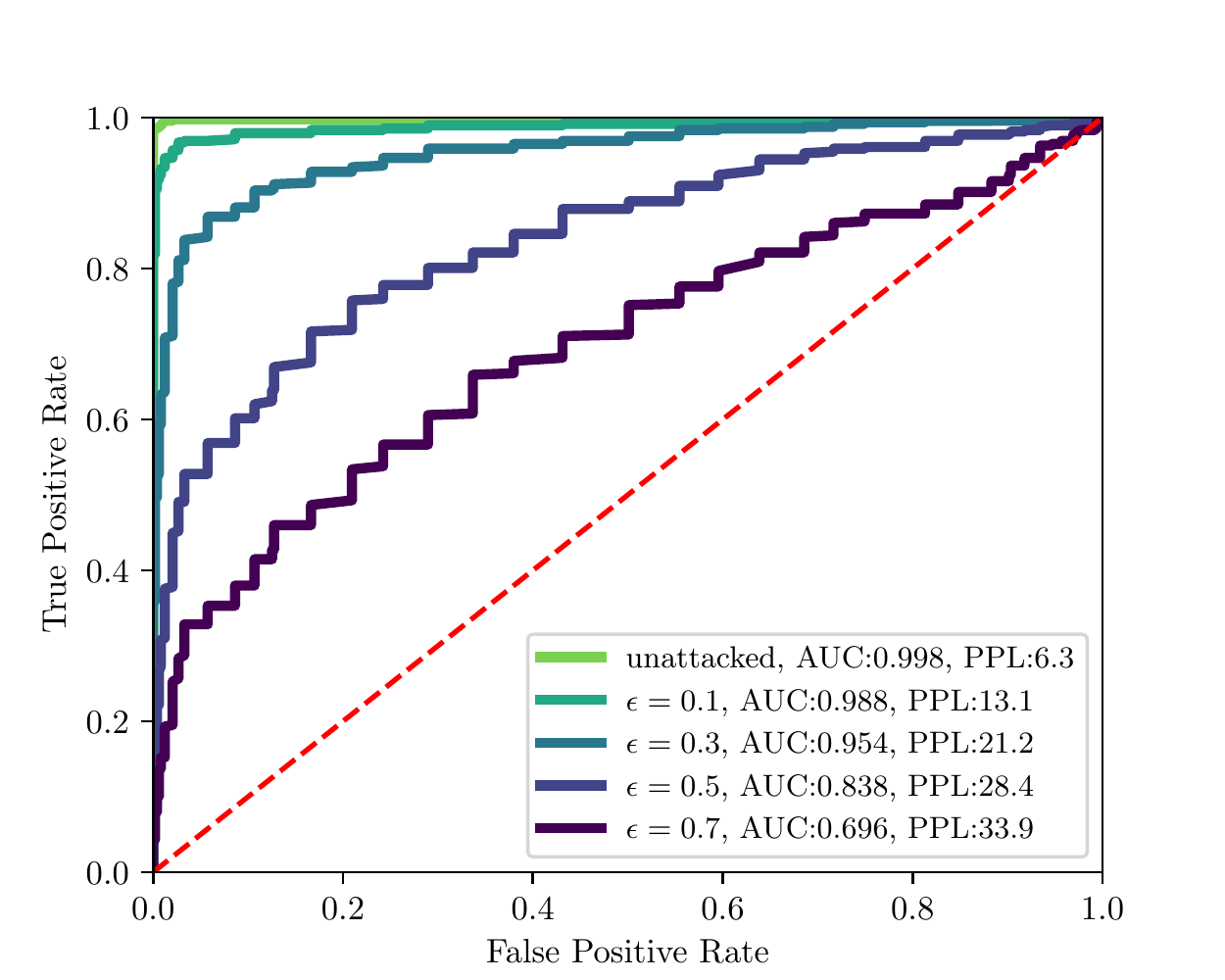}\label{fig:roc-auc-attack-sampling}
    \vspace{-.2cm}
    \caption{ROC curves for watermark detection under attack via the T5 Attack detailed in Section~\ref{sec:replacement-attack}, with various replacement budgets $\varepsilon$. The initial, unattacked watermark is a $\gamma=0.5$, $\delta=2.0$ soft watermark generated using multinomial sampling. The attack achieves a high level of watermark degradation, but \textit{only} at $\varepsilon=0.3$, which \textit{costs} the attacker an average of $\sim15$ points of perplexity compared the PPL of the original watermarked text.}
    \label{fig:roc-auc-attack}
    \vspace{-.5cm}
\end{figure}

\subsection{Degradation Under Attack: Span Replacement Using a LM}\label{sec:replacement-attack}

We study a realistic black-box attack by attempting to remove the presence of the watermark by replacing spans in the original output text using another language model. We treat the watermark algorithm as if it is private, mocking seclusion behind an API.  The attacker does not have access to the locations of green list tokens and instead tries to modify the text through token replacement at random indices until a certain word replacement \emph{budget}, $\varepsilon$, is reached. The budget constraint maintains a level semantic similarity between the original watermarked text and the attacked text, otherwise the ``utility'' of the original text for its intended task may be lost. Also, each span replacement in the attack is performed via inference using a multi-million parameter language model. While this is roughly a third the size of the target model, it means that the attack incurs an associated cost per step implying that a base level of efficiency with respect to model calls would be desired in practice.

In our experiment, we adopt T5-Large \citep{raffel_exploring_2020} as the replacement model and iteratively select and replace tokens until the attacker either reaches the budget, or no more suitable replacement candidates are returned. 

\textbf{Details of the T5 Span Attack.}
We tokenize the watermarked text using the T5 tokenizer. Then, while fewer than $\varepsilon T$ successful replacements have been performed or a maximal iteration count is reached:
\begin{enumerate}[topsep=0pt,itemsep=-0ex,partopsep=1ex,parsep=1ex]
\item Randomly replace one word from the tokenization with a \texttt{<mask>}.
\item Pass the region of text surrounding the \texttt{mask} token to T5 to obtain a list of $k=20$ candidate replacement token sequences via a $50$-way beam search, with associated scores corresponding to their likelihood.
\item Each candidate is decoded into a string. If one of the $k$ candidates returned by the model is \emph{not} equal to the original string corresponding to the masked span, then the attack succeeds, and the span is replaced with the new text.
\end{enumerate}

After attacking a set of $500$ sequences of length $T=200\pm5$ token sequences this way, we compute updated $z$-scores and tabulate error rates (\cref{tab:attacked-acc} in the Appendix). We also generate ROC plots for a range of $\varepsilon$ budgets. While this attack is effective at increasing the number of red list tokens in the text, as shown in \cref{fig:roc-auc-attack}, we only measure a decrease in watermark strength of $0.01$ AUC when $\varepsilon=0.1$. While the watermark removal is more successful at a larger budget of $0.3$, the average PPL of attacked sequences increases by $3\times$ in addition to requiring more model calls.

\section{Related Work} 

\looseness -1 The idea of watermarking, defined as unseen modifications to data that hide identifying information, has a long history. 
However, watermarking of digital text has been considered challenging in the past, due to its discrete nature \citep{katzenbeisser_information_2000}. Watermarking is considered easier for continuous-valued data, where watermarks can be encoded with a variety of well-studied strategies  \citep{petitcolas_information_1999,zhu_hidden_2018,lu_large-capacity_2021,boenisch_systematic_2021}.

In the following, we note that \textit{watermarking}, as a method that encodes enough information to identify the source of a text fragment, is strictly a subset of \textit{steganography}, the task of embedding arbitrary hidden information into data.

\textbf{Watermarking Natural Language.}
\looseness -1 Early approaches to watermarking natural text in digital form in \citet{atallah_natural_2001,atallah_natural_2003} pose a similar problem with similar desiderata as in our setting, except targeted towards classical models. Given a string of text $s$, \citet{atallah_natural_2001} propose to generate text $s'$ with the properties that $s'$ has similar meaning, contains a watermark with an extremely small false-positive rate that is not readable by a party without knowledge of the secret key that generated the watermark, is hard to remove through editing of $s'$ and is further detectable without knowledge of $s$ (or the scheme generating $s$). The actual steganography scheme described therein is limited by its rule-based understanding of natural text to modifications of parsed syntactic tree structures. Finally, the watermark can be read by reconstructing the tree structure, with the chance of a false-positive for a watermark of $w$ bits vanishing quickly at $2^{-w}$.

Rule-based watermarks were further developed in a series of works \citep{chiang_natural_2004,topkara_natural_2006,topkara_hiding_2006,meral_natural_2009,venugopal_watermarking_2011} with variants also embedding watermarks based on synonym tables instead of only parse trees. Early developments were summarized in \citet{jalil_review_2009}, but strong watermarks significantly degraded the text quality due to the limited flexibility of language models at the time.

While approaches via hierarchical language models in \citet{wilson_linguistic_2014} still required human interactions, the emergence of modern neural language models \citep{vaswani_attention_2017,devlin_bert_2019} also allowed for improved watermarking/steganography \citep{fang_generating_2017,ziegler_neural_2019,dai_towards_2019,he_protecting_2022-1,he_cater_2022}. \citet{fang_generating_2017} propose such a neural steganography approach where, to encode a message of $w$ bits, the message is first separated into blocks of length $b$. Then,  the vocabulary $V$ of a language model is partitioned at random into disjoint sets of size $|V|/ 2^b$. A generative LM can then encode the message by generating only a token from the ``allowed'' set at each position. However, this hard rule reduces the quality of generated text, boxing the LM into only a small selection of valid tokens at every step. Other approaches, such as \citet{ueoka_frustratingly_2021} use mask-infilling models such as BERT to edit already-generated text for the purpose of steganography. Finally, \citet{abdelnabi_adversarial_2021} design an end-to-end system where both encoding and decoding are handled by text-to-text language models that are trained adversarially.

With similar motivation to our proposal, \citet{kaptchuk2021meteor} constructs a framework that adapts traditional public-key cryptographic steganography specifically for ``natural'' comunication channels like text using generative models. However, their method, Meteor, relies on a synchronized model framework where the sender and receiver agree on a shared generative model used to embed and decode the hidden bits of information being sent.

Recently, \citet{aaronson_my_2022} announced that he is studying cryptographic approaches to watermarking in collaboration with OpenAI. Their preliminary method is based only on biasing of the LM output, as opposed to complete determination as in \citet{fang_generating_2017}.  While details are not currently available, the description suggests that hashing of $n$-gram sequences is involved. We hope to extend our comparison to this work when more information becomes available.

Note that a separate line of work investigates watermarking \textit{model parameters themselves}. This would not be used to watermark model output (as in this work), but to defend against model stealing \citep{adi_turning_2018,boenisch_systematic_2021}. Approaches, such as \citet{gu_watermarking_2022}, implant backdoor triggers through a finetuning process to cause biased responses to specific inputs, a behavior detectable at verification time.

In contrast to other currently published works, we want to focus on strategies that are simultaneously minimally restrictive to a language model, leverage the LMs own understanding of natural text, require no usage of the LM to decode the watermark, and can be theoretically analyzed and validated.

\textbf{Post-hoc Detection.}
\looseness -1 An alternative to watermarking is to develop detection models that perform a post-hoc analysis of machine-generated text, for example using language model features or finetuning existing large language models to behave as detectors \citep{zellers_defending_2019,tan_detecting_2020}, see an overview in \citet{jawahar_automatic_2020}. These detectors work because LMs still leave detectable signals in generated text. Implementation details, such as sampling strategies, can be reverse-engineered from text \citep{tay_reverse_2020}. However, detection approaches are slowly losing ground as LM capabilities increase, for example \citet{gambini_pushing_2022} note that a range of detection strategies for GPT-2 already struggle with GPT-3. Further, known detectors are also vulnerable to adversarial attacks that degrade their functionality \citep{wolff_attacking_2022}.

While efforts to provide strong detectors continue, as in \citet{tian_gptzero_2023}, ultimately language model progress may make detection infeasible. All post-hoc detection methods require the LM to be significantly biased away from human text in some measurable way, such as low variation in perplexity across sentences \citep{tian_gptzero_2023}. Even for current LLMs, this margin might be small. This is already problematic, as detection schemes that operate within this small margin are susceptible to labeling human text as false-positive, a concern that is especially pressing for people who produce unusual text, such as a non-native speakers, and people who use computer tools to assist them in writing. 
Such populations might be especially at risk for false-positives, which could lead to academic problems if these detectors are used in schools \citep{butoi_things_2023}. 

The watermarking scheme we propose is designed so that false positives are statistically improbable, regardless of the writing patterns of any given human.

\vspace{-.2cm}
\section{Conclusion}

The presented watermark has a number of nice properties that make it a practical choice: 
the watermark is computationally simple to verify without access to the underlying model, false positive detections are statistically improbable, and the watermark degrades gracefully under attack. Further, the proposed scheme can be retro-fitted to any existing model that generates text via sampling from a next token distribution, without retraining. Note, however, that careful implementation and instruction tuning against generative attacks may be required for very large models. 

There is one more important property of the proposed method that we have not discussed: The $z$-statistic  used to detect the watermark depends only on the green list size parameter $\gamma$ and the hash function for generating green lists. There is no dependence on $\delta$ or any other factor related to how the green list is enforced. 
For this reason, one can deploy the watermark using context-specific $\delta$ choices or green list enforcement rules for different kinds of text (e.g., prose vs code, or small vs large models) while using the same downstream watermark detector.  One can also change a proprietary implementation of the watermarked sampling algorithm without any need to change the detector.
Finally, the watermarking method could be turned on only in certain contexts, for example when a specific user seems to exhibit suspicious behavior.

There are still a number of remaining open questions regarding watermarking.  For example, what kind of robust hashing rules are possible, and when are these rules provably optimal?  What is the best way to test for the watermark in a streaming context, or in a context where a short span of watermarked text lives inside a longer non-watermarked span?  Are there simple sensitivity bounds that are more accurate than those presented above for large $\delta$ and small $\gamma$?
We hope our present results are enough to convince readers that watermarks could be a practical tool for combating malicious uses of generative models, and we leave these additional questions for future research

\vspace{-.2cm}
\section{Acknowledgements}
This work was made possible by the ONR MURI program, DARPA GARD (HR00112020007), the Office of Naval Research (N000142112557), and the AFOSR MURI program.  Commercial support was provided by Capital One Bank, the Amazon Research Award program, and Open Philanthropy. Further support was provided by the National Science Foundation (IIS-2212182), and by the NSF TRAILS Institute (2229885).

\label{lastpagemaintext}
\icmltitlerunning{A Watermark for Large Language Models.}
\vspace{-.3cm}
\bibliography{NLP_auto_references,manual_references}
\bibliographystyle{icml2023}

\newpage
\appendix
\onecolumn
\renewcommand{\thepage}{\arabic{page}}
\setcounter{page}{1}
\icmltitlerunning{A Watermark for Large Language Models. Appendix Page \thepage\ of \pageref{lastpageapp}.}

\begin{table*}[ht]
\tiny
\begin{tabular}{p{0.2cm}|p{2.75cm}|p{2.75cm}|p{2.75cm}|p{2.75cm}|p{0.2cm}|p{0.42cm}|p{0.42cm}|p{0.3cm}|p{0.3cm}}
\toprule
 idx &  prompt & real completion &  no watermark (NW) &  watermarked (W) &  $S$ &  (NW) $z$ &  (W) $z$ &  (NW) PPL &  (W) PPL \\
\midrule
1011 &  [...]nned any special treatment as an officer. The living quarters reserved for him -- palatial by camp standards -- remained empty as he bunked nearer to the workers. Straight-talking and cigar-chomping ( &   a habit that earned him the nickname "Big Smoke"), he and politics didn't mix well. Signing on in 1905 "until he could predict success or failure according to his own judgment," Stevens resigned in 19[...truncated] &  he was nicknamed "The Devil"), Stevens was a man who could get anything done. He had one major flaw: He was an expert at getting things done -- without getting anything done. The problem was, he got t[...truncated] &  another Stevens habit), Stevens impressed them so much that he even earned the nickname "Bully Stevens" from a group of workers.\textbackslash nAs the tedious, years-long construction dragged on -- years and years a[...truncated] &               0.87 &           0.14 &          8.20 &       8.77 &     10.33 \\\midrule
 924 & [...] Helen Cavallo.... She's so funny and so warm, and she's the kind of person the kids won't look away from, and the moms won't look away from."\textbackslash nTelevision personality and author of "Courtney's Creative &    Adventurers," Courtney Watkins, hosts "Thinking Up" -- which provides imaginative children's activities, drawing from her eclectic background as a teacher of kindergarten and art and a designer of je[...truncated] &   Child," Courtney Love, is one of the series' executive producers. Gilmore said the Disney team has a "talent for building relationships" with the target audience, who, she said, are "very passionate [...truncated] &   World: The Parental Roadmap for Success," Gilmore also noted that there's an opportunity to harness the power of social media to share stories and provide expert advice.\textbackslash n"We've seen that, too," she s[...truncated] &               0.87 &          -0.28 &          9.62 &       6.95 &     10.30 \\\midrule
1151 & [...]Carrie and a formidable Heather Orth as her frighteningly religious mother.\textbackslash nIt's not Oeschger's fault that there's a big disconnect between the painfully shy outsider she plays so convincingly and the &    forceful notes she sings with dynamic conviction. Cohen's book works best for those who know the story well - no small number - jumping from one plot point to the next without connecting the dots. Ca[...truncated] &  confident, vampish Carrie whose songwriting is in a state of constant revising. It's a shame, because Oeschger is an able actress. And a terrific singer.\textbackslash nThe original production, which had just two a[...truncated] &   awful person inside that she fights so hard to control with forced bravado.\textbackslash nBut the script's missteps and musical shortcomings, especially in the tightly-crafted scenes, leave the main characters far[...truncated] &               0.87 &          -0.99 &          8.77 &       8.40 &     12.33 \\\midrule
 853 & [...]? Transferring misbehaving government servants will not solve any problem. They must be suspended and demoted. This is to ensure they will think twice before doing anything.\textbackslash nAnonymous \#18452573: Bigot &  ry is the by-product of racism. This is just another form of ethnocentrism and xenophobia due to BN's 50 years of racially divided political parties and racist policies.\textbackslash nIslam, Christianity and other [...truncated] & ry in schools is rampant and this is not acceptable. Muslims should not be allowed to enter the schools.\textbackslash nTun: I am saddened and disappointed at how our education system is being used to promote discri[...truncated] & ry against Orang Asli children is a new low. They're not children to be bullied or manipulated by the teachers.\textbackslash nHow come the teachers are the only people doing such acts? Why not the parents?\textbackslash nI read o[...truncated] &               0.87 &           0.99 &         10.47 &       7.42 &     10.55 \\\midrule
 341 & [...]When more responsible people carry, violent crime goes down. The FBI crime statistics don’t lie. Why do the left and gun grabbing crowd ignore this fact? Liberalism is a mental disorder.\textbackslash nLets look one &  poll that this particular article is using for it’s numbers.\textbackslash nThe poll used 2 Senate Districts in the State of VA. District 21 and District 38.\textbackslash nwith the total numbers of respondents being 305 voters i[...truncated] &  more time at the facts.\textbackslash nThe “Fix Gun Checks Act” would make it tougher for the mentally ill to legally buy guns by strengthening background checks so people like James Holmes (the Aurora shooter) can[...truncated] &    day at the number of school shootings that have occurred since Columbine. They were not just random shootings. They were planned. In the preplanned shootings, the students/victims were well trained a[...truncated] &               0.88 &          -0.42 &          9.22 &       1.12 &     13.66 \\
\bottomrule
\end{tabular}
\caption{High spike entropy examples under a $\delta=2.0,\gamma=0.5$ watermark with multinomial sampling.}\label{tab:high-entropy-examples}
\end{table*}

\begin{table*}[t]
\tiny
\begin{tabular}{p{0.2cm}|p{2.75cm}|p{2.75cm}|p{2.75cm}|p{2.75cm}|p{0.2cm}|p{0.42cm}|p{0.42cm}|p{0.3cm}|p{0.3cm}}
\toprule
 idx &  prompt & real completion &  no watermark (NW) &  watermarked (W) &  $S$ &  (NW) $z$ &  (W) $z$ &  (NW) PPL &  (W) PPL \\
\midrule
 132 & [...]cond season at Hall Bros Oval.\textbackslash nThe defender also admitted his surprise at Young’s run to the finals but credited the injection of youth into the side.\textbackslash n“We were really in a building phase last year and &  we copped a few floggings with all those juniors blokes coming in,” Galvin said.\textbackslash n“Now, we’ve kept that core group together for two years and I think we’ve come along quicker than we anticipated.\textbackslash nROCK[...truncated] &  we copped a few floggings with all those juniors blokes coming in,” Galvin said.\textbackslash n“Now, we’ve kept that core group together for two years and I think we’ve come along quicker than we anticipated.\textbackslash n“Tha[...truncated] &                            we copped a few floggings with all those juniors blokes coming in,” Galvin said.\textbackslash n“Now, we’ve kept that core group together for two years and I think we’ve come along quicker than we anticipated.\textbackslash n“Tha[...truncated] &               0.58 &          -1.27 &         -1.13 &       1.05 &      1.04 \\\midrule
 850 &  [...]son tear through the banner Friday, Nov. 16, in Copper Mountain's East Village.\textbackslash nBetty the Yeti rallies the first chairlift riders of the 2018-2019 ski season Friday, Nov. 16, in Copper Mountain's East &  Village.\textbackslash nCopper Ski Resort Chief Operations Officer Jesse True high fives first chairlift riders of the 2018-2019 ski season Friday, Nov. 16, in Copper Mountain's East Village.\textbackslash nCopper Mountain reside[...truncated] &  Village.\textbackslash nCopper Mountain Resort will open Friday with 275 acres of skiable terrain based out of East and West villages.\textbackslash nThe ski area will open on a limited schedule on Friday, Nov. 16, followed by a [...truncated] &                           Village.\textbackslash nA rider heads down a slope on the opening day of 2018-2019 season at Copper Mountain.\textbackslash nA rider heads down a slope on the opening day of 2018-2019 season at Copper Mountain.\textbackslash nA rider heads down[...truncated] &               0.61 &           0.00 &          5.37 &       4.58 &      1.18 \\\midrule
 939 &  [...]ngs.\textbackslash nAnd now we have Trump calling on Bee to be fired. You know who else hates comedy at their expense? Authoritarian leaders. Egypt’s President Abdel Fattah el-Sisi, a man Trump loves, banned Egypt’s &   version of The Daily Show because of the way its comedy mocked the government. The show’s star, Bassem Youssef, had to flee the country and is now living in exile with his family.\textbackslash nTurkey’s strongman [...truncated] &   top comedy show after the host criticized the military’s coup over a year ago. And then President Vladimir Putin banned comedy on Russian television over the same topic.\textbackslash nIt’s not hard to see why Trum[...truncated] &                             leading comedy show and radio host, Ahmed Mansour, from entering the country.\textbackslash nTrump’s attacks on freedom of expression go far beyond just news media. Trump also wants to silence those in the entertai[...truncated] &               0.61 &           0.99 &          2.83 &       4.86 &      1.27 \\\midrule
 444 &  [...]1 on the season with a 2.76 GAA and.906 save percentage.\textbackslash nMar. 11 5:28 PM PT6:28 PM MT7:28 PM CT8:28 PM ET0:28 GMT8:28 5:28 PM MST6:28 PM CST7:28 PM EST4:28 UAE (+1)20:28 ET21:28 BRT - Sergei Bobrovsky &   stopped 18 of 19 shots in the Blue Jackets' 2-0 loss to the Islanders on Monday. The defeat dropped him to 29-22-1 on the season with a 2.74 GAA and.907 save percentage.\textbackslash nMar. 9 5:41 PM PT6:41 PM MT7:[...truncated] &   stopped 28 of 30 shots in the Blue Jackets' 3-2 shootout loss to the Islanders on Sunday. The loss dropped him to 30-22-1 on the season with a 2.76 GAA and.906 save percentage.\textbackslash nMar. 10 6:15 PM PT7:15[...truncated] &                             stopped 30 of 34 shots in the Blue Jackets' 6-2 win over the Canadiens on Monday. The victory pushed him to 31-21-1 on the season with a 2.72 GAA and.908 save percentage.\textbackslash nMar. 10 11:49 AM PT12:49 PM [...truncated] &               0.62 &          -0.99 &          2.40 &       1.33 &      1.45 \\\midrule
1171 &  [...]South Elliott Place, near Lafayette Street, Fort Greene; \$; no smoking; Mastercard and Visa.\textbackslash n(718) 857-8863; 620 Vanderbilt Avenue, Prospect Heights; \$\$; no smoking; American Express, Mastercard, Visa & .\textbackslash n(718) 624-9267; 218 Court Street, near Warren Street, Cobble Hill; \$; no smoking; American Express, Mastercard and Visa.\textbackslash n(718) 499-5557; 426A Seventh Avenue, near 15th Street, Park Slope; \$; no smok[...truncated] & .\textbackslash n(718) 857-8863; 620 Vanderbilt Avenue, Prospect Heights; \$\$; no smoking; American Express, Mastercard, Visa.\textbackslash n(718) 857-8863; 620 Vanderbilt Avenue, Prospect Heights; \$\$; no smoking; American Express[...truncated] & .\textbackslash n\textbackslash nBusiness in the Park\textbackslash n\textbackslash nHaley’s, 77\textbackslash n\textbackslash nHaley’s Restaurant, 77\textbackslash n\textbackslash nHaley’s, 77\textbackslash n\textbackslash nHaley’s, 77\textbackslash n\textbackslash nHaley’s, 77\textbackslash n\textbackslash nHaley’s, 77\textbackslash n\textbackslash nHaley’s, 77\textbackslash n\textbackslash nHaley’s, 77\textbackslash n\textbackslash nHaley’s, 77\textbackslash n\textbackslash nHaley’s, 77\textbackslash n\textbackslash nHaley’s, 77\textbackslash n\textbackslash nHaley’s, 77\textbackslash n\textbackslash nHaley’s,[...truncated] &               0.62 &          -0.71 &         11.17 &       1.09 &      1.48 \\
\bottomrule
\end{tabular}
\caption{Low spike entropy examples under a $\delta=2.0,\gamma=0.5$ watermark with multinomial sampling.}\label{tab:low-entropy-examples}
\end{table*}

\begin{table*}[t]
\tiny
\begin{tabular}{p{0.2cm}|p{2.75cm}|p{2.75cm}|p{2.75cm}|p{2.75cm}|p{0.2cm}|p{0.42cm}|p{0.42cm}|p{0.3cm}|p{0.3cm}}
\toprule
 idx &  prompt & real completion &  no watermark (NW) &  watermarked (W) &  $S$ &  (NW) $z$ &  (W) $z$ &  (NW) PPL &  (W) PPL \\
\midrule
1105 & [...]ent to mark 80 important moments in our shared history. Looking back it's easy to see these moments changed the way Australians lived and thought.\textbackslash nThe Opera House Project is an online documentary that &  tells the story behind one of the greatest buildings of the twentieth century and explores the cultural heritage that has ensued for more than forty years to the present day.\textbackslash nA selection of archival [...truncated] &    uses archive footage to explore the history of the Opera House. It's the first and only full-length documentary on the Opera House, and the production is helmed by 90s documentary maker and broadcast[...truncated] &     explores Australia's National Art Gallery. It tells the history of the construction, and evolution of a work that has been called one of the most significant cultural legacies of Australian art.\textbackslash nSydn[...truncated] &               0.78 &           0.85 &         11.46 &       3.40 &      5.22 \\\midrule
 354 & [...]epeatedly to first see the group list, then the correct usage of the commands.\textbackslash nwill turn the Tool Tips on. Using the value 0 in front of the above command will turn the Tool Tips off. This is shown as &   1/0 below (called a Boolean value - [bool]), which indicates that using 1 enables a function, 0 disables it. The use of \# (or [float]) after a command below means a numeric value is required for the [...truncated] &  the command "tooltips 0" in the Console.\textbackslash n\textbackslash nThere are a range of useful commands which can be accessed through the game console, or entered into the.con files. To open the console, press the tilde (\textasciitilde ) [...truncated] &  the tooltip to the right of this group list.\textbackslash n\textbackslash nA list of all commands can also be displayed by pressing the \textasciitilde  key repeatedly, typing the commands and then pressing the TAB key repeatedly, again.\textbackslash n\textbackslash nA li[...truncated] &               0.71 &          -0.42 &         11.46 &       1.59 &      3.00 \\\midrule
 808 & [...]called it Kevin Hart Presents Chris Rock’s “Top Five With Kevin Hart.” Hillary should announce Kevin as her running mate and be done with it.\textbackslash nAfter he makes you laugh for an hour, you feel like you’re &   his friend—you feel his joy and his pain. If he has asthma, you have asthma. That’s because Kevin’s vulnerable. When he’s talking about his kids, his kids know more than him. He’s not afraid to be “f[...truncated] &    in bed with him, and you’re ready to go to sleep. I met Kevin the night before the show, and he came up to my dressing room and hugged me and said, “This is the best show of the year.” He was just li[...truncated] &     the luckiest man in the world. And then just like that, Kevin Hart makes you feel like a complete moron.\textbackslash nI just finished playing an iPhone game called You’re the Best Dad Ever, and I was trying to gi[...truncated] &               0.81 &           3.68 &         11.74 &       3.75 &      5.17 \\\midrule
 438 & [...] logged into their Google account and have verified profiles that match queries for the site.\textbackslash nGoogle's John Mueller said there is no ranking benefit in using different Google Search Console and Google &   Analytics accounts for each individual web site you manage. The topic came up before, as long as you are not spamming Google - there also is no down side to using the same accounts across multiple we[...truncated] &   Analytics data to calculate ranking sentiment.\textbackslash nGoogle have a new feature called Keyword Difficulty Analysis that shows you the keywords your competitors are ranking for. It shows the demand curve as [...truncated] &   + accounts to see different numbers.\textbackslash nGoogle also released their own great blog post on the news algorithm. They give lots of great advice to help your site do better.\textbackslash nFinally, at the end of September [...truncated] &               0.77 &           1.41 &         11.88 &       9.24 &      4.10 \\\midrule
 315 & [...]tled out of court and publicly reconciled.\textbackslash nIn the ’80s the band’s popularity waned in the United States but remained strong abroad. Robin released three solo albums, with limited success. The Bee Gees &  returned with some moderate hits in the late 1990s and were inducted into the Rock and Roll Hall of Fame in 1997. With his brothers, Mr. Gibb won six Grammys.\textbackslash nIn addition to his wife and his brother [...truncated] &   continued to tour, and Barry became a television producer.\textbackslash nBut in the early ’90s, the Bee Gees’ popularity remained high. They scored a hit with “Don’t Stop Believing” in 1990, and in 1992 the Bee Ge[...truncated] &    ’ 1990 album, “Spirits of the Century,” was a mixed critical and commercial success.\textbackslash nWhen the brothers were nominated for a Grammy Award in 1990, Mr. Gibb’s “You Should Be Dancing” and “Massachusetts,[...truncated] &               0.68 &           2.97 &         12.73 &       3.15 &      1.93 \\
\bottomrule
\end{tabular}
\caption{High $z$-score examples under a $\delta=2.0,\gamma=0.5$ watermark with multinomial sampling.}\label{tab:high-z-examples}
\end{table*}

\begin{table*}[t]
\tiny
\begin{tabular}{p{0.2cm}|p{2.75cm}|p{2.75cm}|p{2.75cm}|p{2.75cm}|p{0.2cm}|p{0.42cm}|p{0.42cm}|p{0.3cm}|p{0.3cm}}
\toprule
 idx &  prompt & real completion &  no watermark (NW) &  watermarked (W) &  $S$ &  (NW) $z$ &  (W) $z$ &  (NW) PPL &  (W) PPL \\
 \midrule
 132 & [...]cond season at Hall Bros Oval.\textbackslash nThe defender also admitted his surprise at Young’s run to the finals but credited the injection of youth into the side.\textbackslash n“We were really in a building phase last year and &  we copped a few floggings with all those juniors blokes coming in,” Galvin said.\textbackslash n“Now, we’ve kept that core group together for two years and I think we’ve come along quicker than we anticipated.\textbackslash nROCK[...truncated] &  we copped a few floggings with all those juniors blokes coming in,” Galvin said.\textbackslash n“Now, we’ve kept that core group together for two years and I think we’ve come along quicker than we anticipated.\textbackslash n“Tha[...truncated] &  we copped a few floggings with all those juniors blokes coming in,” Galvin said.\textbackslash n“Now, we’ve kept that core group together for two years and I think we’ve come along quicker than we anticipated.\textbackslash n“Tha[...truncated] &               0.58 &          -1.27 &         -1.13 &       1.05 &      1.04 \\\midrule
 444 &  [...]1 on the season with a 2.76 GAA and.906 save percentage.\textbackslash nMar. 11 5:28 PM PT6:28 PM MT7:28 PM CT8:28 PM ET0:28 GMT8:28 5:28 PM MST6:28 PM CST7:28 PM EST4:28 UAE (+1)20:28 ET21:28 BRT - Sergei Bobrovsky &   stopped 18 of 19 shots in the Blue Jackets' 2-0 loss to the Islanders on Monday. The defeat dropped him to 29-22-1 on the season with a 2.74 GAA and.907 save percentage.\textbackslash nMar. 9 5:41 PM PT6:41 PM MT7:[...truncated] &   stopped 28 of 30 shots in the Blue Jackets' 3-2 shootout loss to the Islanders on Sunday. The loss dropped him to 30-22-1 on the season with a 2.76 GAA and.906 save percentage.\textbackslash nMar. 10 6:15 PM PT7:15[...truncated] &   stopped 30 of 34 shots in the Blue Jackets' 6-2 win over the Canadiens on Monday. The victory pushed him to 31-21-1 on the season with a 2.72 GAA and.908 save percentage.\textbackslash nMar. 10 11:49 AM PT12:49 PM [...truncated] &               0.62 &          -0.99 &          2.40 &       1.33 &      1.45 \\\midrule
 302 &  [...] rating. The price target for Sol Gel Technologies is set to \$20. Sol Gel Technologies shares closed at \$11.13 on Friday.\textbackslash nAnalysts at Wells Fargo initiated coverage on Landmark Infrastructure Partners &   LP (NASDAQ: LMRK) with a Market Perform rating. The price target for Landmark Infrastructure Partners is set to \$18.50. Landmark Infrastructure Partners closed at \$17.75 on Friday.\textbackslash nBarclays initiated[...truncated] &   L.P. (NASDAQ: LMRK) with an Overweight rating. The price target for Landmark Infrastructure Partners is set to \$12. Landmark Infrastructure Partners shares closed at \$10.02 on Friday.\textbackslash nAnalysts at Jef[...truncated] &  , L.P. (NASDAQ: LMRK) with an Overweight rating. Landmark Infrastructure Partners shares rose 7.39 percent to close at \$22.75 on Friday.\textbackslash nWells Fargo initiated coverage on Freshpet Inc. (NASDAQ: FRPT) [...truncated] &               0.66 &          -2.55 &          2.83 &       1.75 &      2.08 \\\midrule
 482 &  [...]nika Aigner, with sister Elisabeth as her guide, sprang a surprise in women’s visually impaired event on the fourth day of the World Para Alpine Skiing World Cup at the Spanish resort of La Molina.\textbackslash nSw &   itzerland's Theo Gmur, Paralympic champion in the men's standing giant slalom, succeeded at the third attempt in beating France's 18-year-old world champion Arthur Bauchet at the World Para Alpine Ski[...truncated] &  eden’s Chris Vos won gold in the men’s super-G at the World Para Snowboard World Cup Finals in Klövsjö in Sweden.\textbackslash nThe final day of action in Klövsjö concludes today with the men’s super-G and women’s [...truncated] &  eden’s Chris Vos clinched gold as the World Para Snowboard World Cup Finals in Klövsjö concluded on Friday with the main event for men’s and women’s single and double slalom.\textbackslash nKlövsjö is set to host th[...truncated] &               0.70 &          -0.71 &          2.83 &       2.86 &      3.34 \\\midrule
 939 &  [...]ngs.\textbackslash nAnd now we have Trump calling on Bee to be fired. You know who else hates comedy at their expense? Authoritarian leaders. Egypt’s President Abdel Fattah el-Sisi, a man Trump loves, banned Egypt’s &   version of The Daily Show because of the way its comedy mocked the government. The show’s star, Bassem Youssef, had to flee the country and is now living in exile with his family.\textbackslash nTurkey’s strongman [...truncated] &   top comedy show after the host criticized the military’s coup over a year ago. And then President Vladimir Putin banned comedy on Russian television over the same topic.\textbackslash nIt’s not hard to see why Trum[...truncated] &   leading comedy show and radio host, Ahmed Mansour, from entering the country.\textbackslash nTrump’s attacks on freedom of expression go far beyond just news media. Trump also wants to silence those in the entertai[...truncated] &               0.61 &           0.99 &          2.83 &       4.86 &      1.27 \\
\bottomrule
\end{tabular}
\caption{Low $z$-score examples under a $\delta=2.0,\gamma=0.5$ watermark with multinomial sampling.}\label{tab:low-z-examples}
\end{table*}

\section{Experimental Details}

\subsection{Sample Outputs}\label{sec:sample-outputs}

We provide series of representative outputs from different ranges in the sample space for model generations under a soft watermark with parameters $\delta=2.0,\gamma=0.5$ under the multinomial sampling scheme. To tabulate these outputs, the $\sim 500$ generations collected at this setting are either sorted by the average spike entropy of the watermarked model's output distribution at generation time, or the measured test statistic, the $z$-score for that sequence. The top and bottom $5$ samples according to these orderings are shown for both entropy (\cref{tab:low-entropy-examples} and \cref{tab:high-entropy-examples}) and $z$-score (\cref{tab:low-z-examples} and \cref{tab:high-z-examples}).

\subsection{Measuring Perplexity: Oracle Language Model}\label{sec:measuring-ppl}

To compute perplexity, the larger, \textit{Oracle} Language Model is fed the original prompt as input, as described in the main body, and perplexity is computed via taking the exponential of the average token-wise loss according to the oracle's next token distribution at every output index. Note that loss is computed for \textit{only} the generated tokens produced by either a watermarked or non-watermarked model.

\begin{table}
\small 
\begin{tabular}{@{}c|l|l@{}}
\multicolumn{1}{l}{}                         & \multicolumn{2}{c}{Text is watermarked, i.e. machine-generated}                                                                                              \\ \midrule
\multirow{2}{*}{Hypothesis Test is Rejected} & TP - Watermarked text is correctly flagged.                                & FP -  Text that is not watermarked is flagged. \\
                                             & FN - Text is watermarked, but cannot be detected. & TN - Text is not watermarked and is not flaggged.            
\end{tabular}
\caption{Reference table for possible outcomes of the hypothesis test. Type-I errors, false positives, are improbable by construction of the watermarking approach, but type-II errors, false negatives, appear naturally for low-entropy sequences that cannot be watermarked.}
\label{tab:explanation_for_dummies}
\end{table}
\section{Detailed Threat Model}

For completeness, we formally define the threat model for the attacks discussed in \cref{sec:attacks} here. As described, attacks may occur when malicious users operate bots/sock-puppets on social media, try to fool a CATPCHA, or complete an academic assignment  \citep{foltynek_academic_2019}. 
In this work we formally \textit{define adversarial behavior as all efforts by a party the using machine-generated text to remove the watermark}.
It is ultimately important to remember that we describe a watermark only on the tokens of the generated text, i.e. on its form and style, and not on its semantic content. For example, a completely new essay written based on an outline or initial draft provided by a LM could not be detected. Such semantic watermarks may be possible, but we do not study this setting here.

\paragraph{Threat Model.}
We assume two parties, a model owner providing a text generation API, and an attacker attempting to remove the watermark from the API output. The attacker moves second, and is aware that the API contains a watermark. In public mode, the attacker is aware of all details of the hashing scheme and initial seed. In private mode, the attacker is aware of the watermark implementation, e.g. \cref{alg:private_robust_watermark}, but has no knowledge of the key of the pseudo-random function $F$. The attacker attempts to reduce the number of green-listed occurrences in the text, reducing the $z$-score computed by a defender. In public mode, any party can evaluate the watermark. In private mode, only the model owner can evaluate the watermark and provides a text detection API. We assume that this API is rate-limited. 
We assume the attacker has access to other non-watermarked language models, but these models are weaker than the API under attack. The attacker is allowed to modify the generated text in any way.

Note that removing the watermark is always a trivial task if language model quality is disregarded -- one can simply replace the entire text with random characters.  For this reason only attacks that result in a reasonable language quality trade-off for the attacker are relevant. A defense is hence also successful if any watermark removal by the attacker reduces the quality of generated text to that of generated text achievable using a public model.

\section{Detection Accuracy of Multinomial Sampling}
\label{sec:multinomial_results}
When a multinomial sampler is used (which is assumed by \cref{maintheorem}), we use the softmax output with standard temperature hyperparameter \texttt{temp=0.7}. We analyze the alignment between the empirical strength of the watermark and the theoretical lower bound for $\gamma=.5$ in \cref{fig:analytic}. We find that the theoretical bound is quite tight for smaller values of $\delta,$ but the theorem under-estimates watermark sensitivity for larger $\delta.$

\begin{figure}[h!]
\begin{center}
\includegraphics[width=0.5\columnwidth]{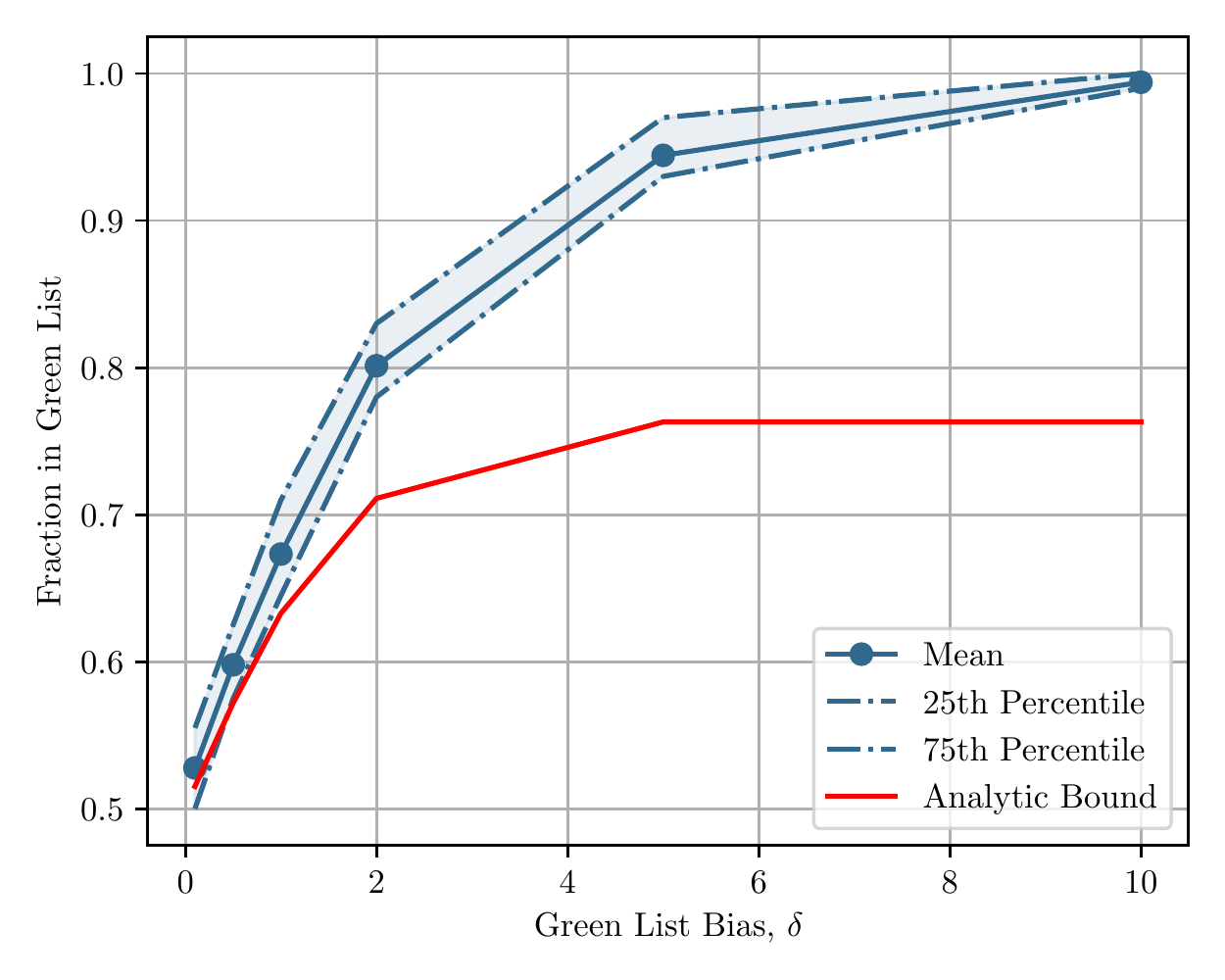}
\caption{Empirical green list fraction vs bias parameter $\delta.$  We compare to the theoretical bound predicted by \cref{maintheorem}.}
\label{fig:analytic}
\end{center}
\vskip -0.2in
\end{figure}

ROC curves for multinomial sampling, and greedy decoding with 8-way beam search in the 200 token case are depicted in \cref{fig:roc-auc} and \cref{fig:roc-auc-beams} (Subsets of \cref{fig:roc-auc-beams-compressed} from the main work.). Tables with error rates and accuracy numbers at selected $z$ values are provided in \cref{tab:joint-acc}.

\begin{figure*}[h]
    \centering
    \subfigure[]{\includegraphics[width=0.49\textwidth]{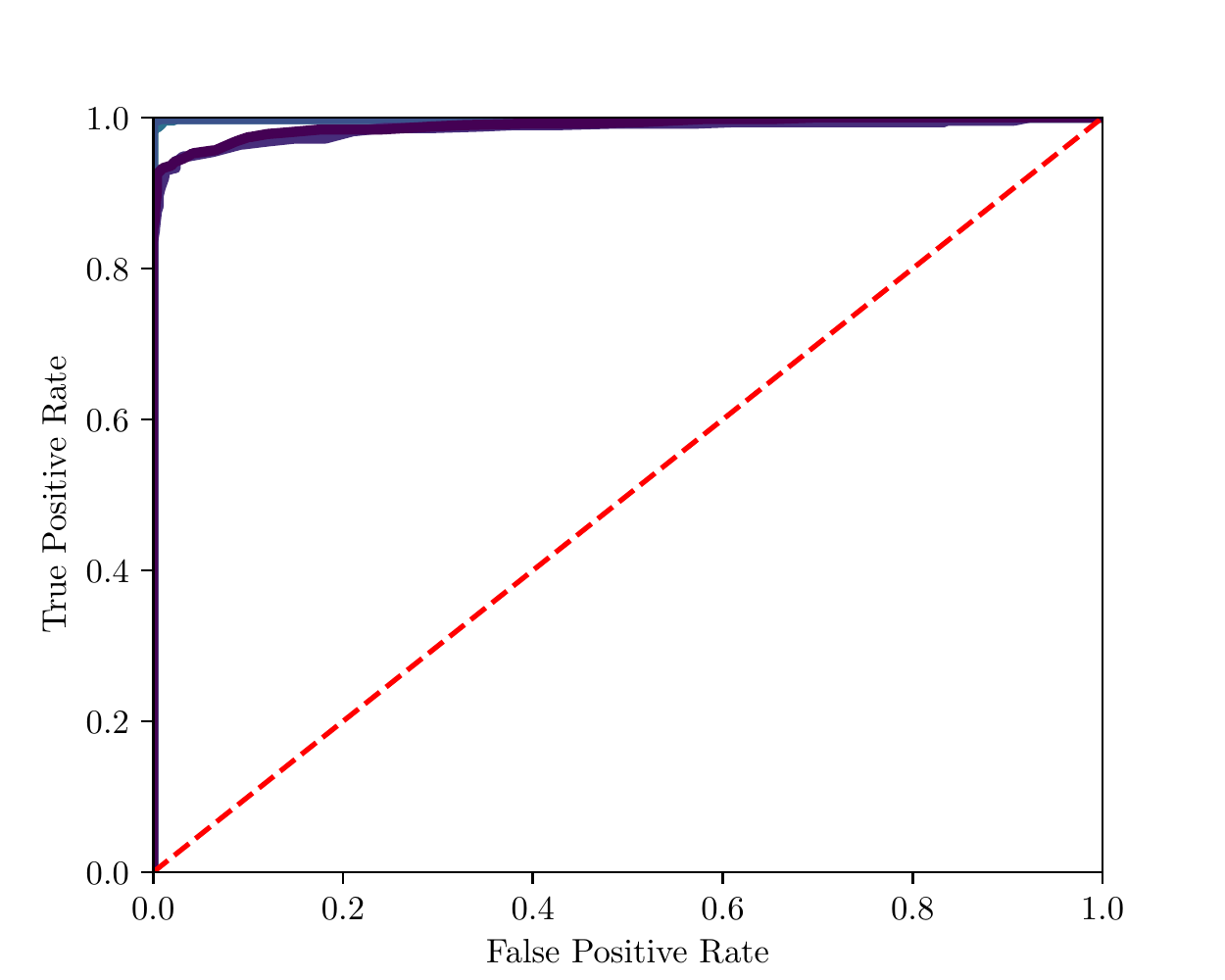}}
    \subfigure[]{\includegraphics[width=0.49\textwidth]{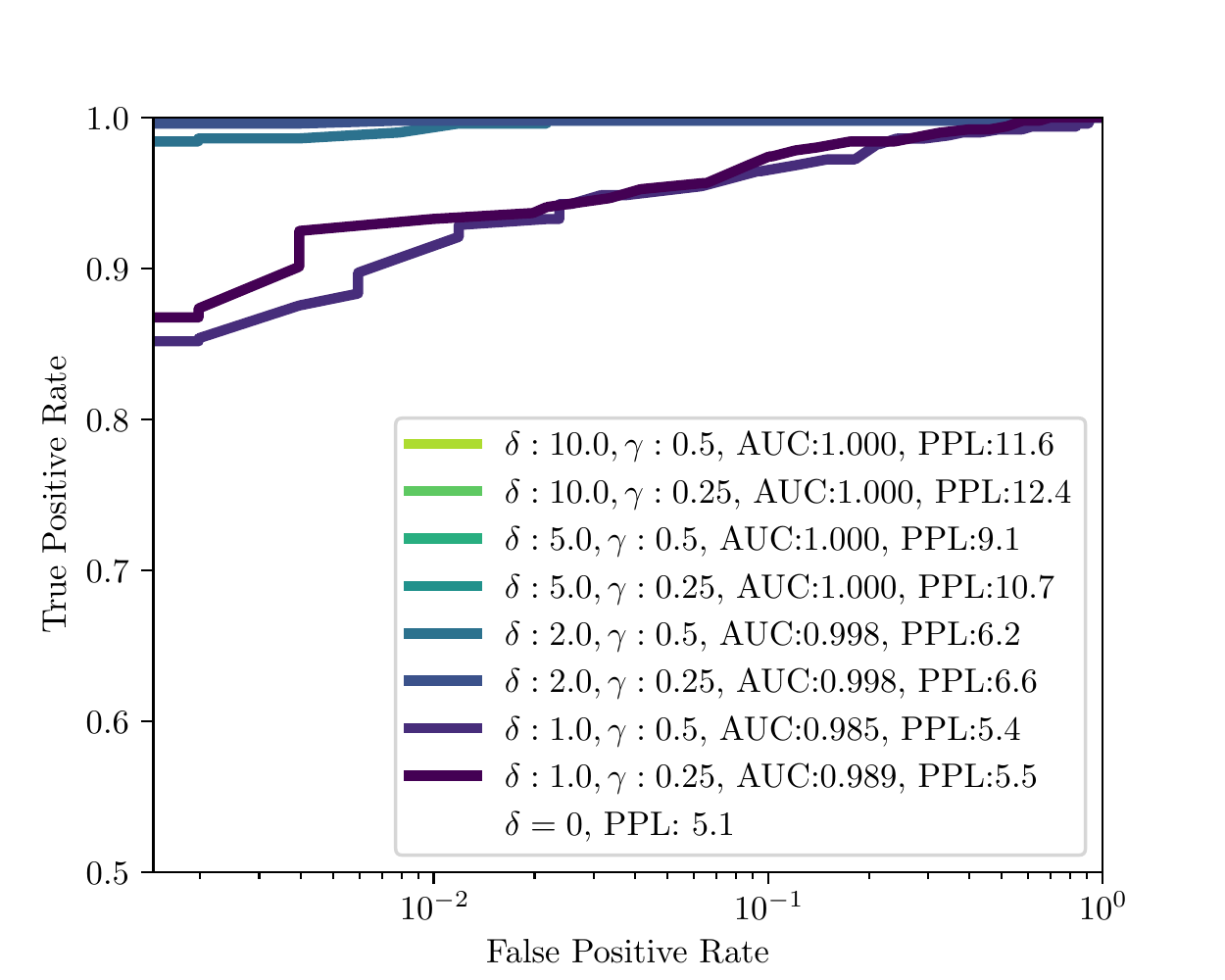}\label{fig:roc-auc-zoom-2}}
    \caption{(a) ROC curve with AUC values for watermark detection. Curves for several choices of watermark parameters $\gamma$ and $\delta$ are shown - multinomial sampling is used across all settings. (b) The same chart, but with different axes to make detail visible. The stronger watermarks corresponding to lower $\gamma$ values and higher $\delta$ values achieve the best error characteristics.}
    \label{fig:roc-auc}
\end{figure*}

\begin{figure*}[h]
    \centering
    \subfigure[]{\includegraphics[width=0.49\textwidth]{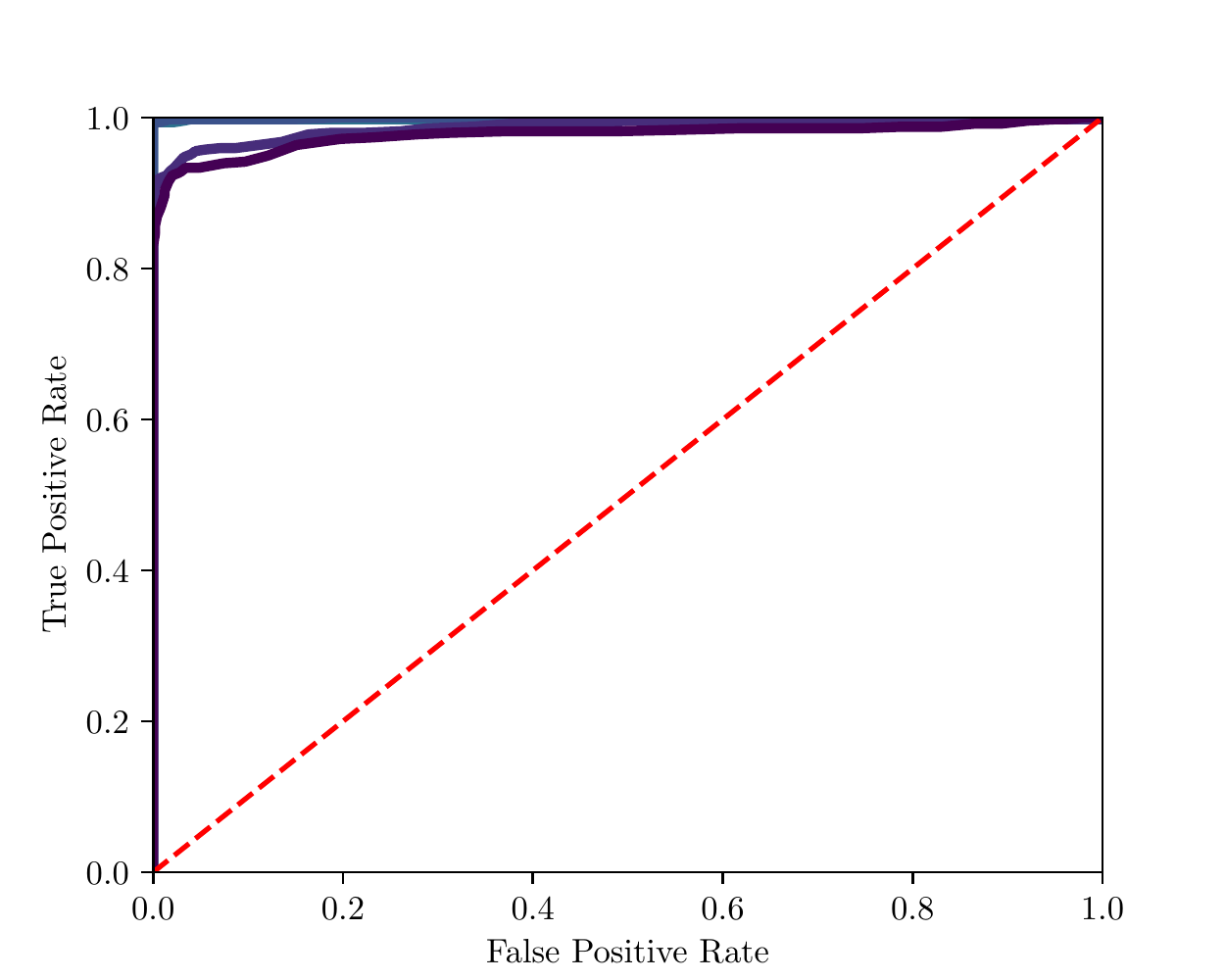}
    }
    \subfigure[]{\includegraphics[width=0.49\textwidth]{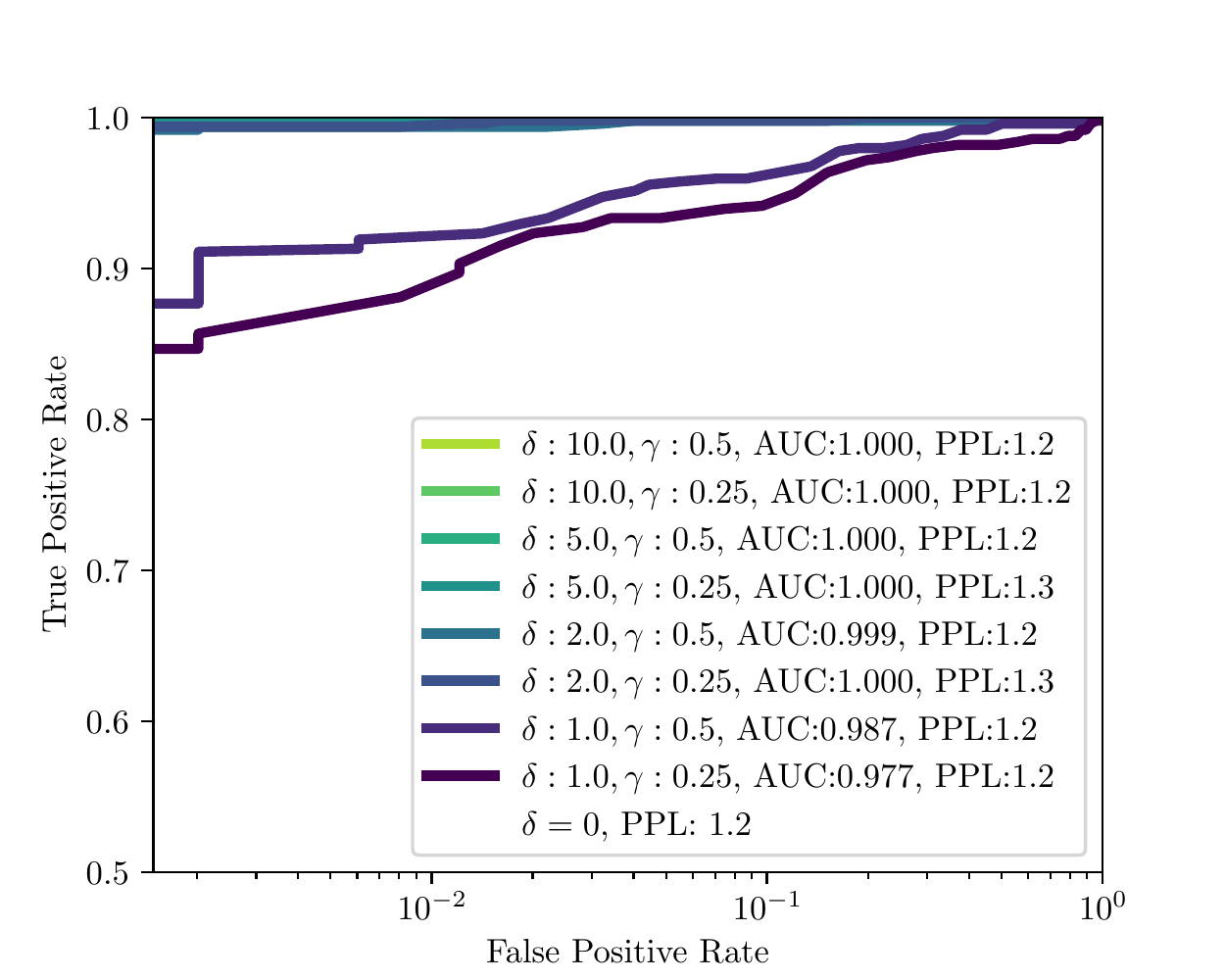}\label{fig:roc-auc-beams-zoom-2}}
    \caption{(a) ROC curve with AUC values for watermark detection. Curves for several choices of watermark parameter $\delta$ are shown - greedy decoding and 8-way beam search is used to generate tokens in all settings. (b) The same chart, but with different axes to make detail visible. Similarly to \cref{fig:roc-auc}, higher $\delta$ values achieve stronger performance, but additionally we see that for a given $\delta$, the beam search allows the watermark to capture slightly more AUC than the corresponding parameters under the multinomial sampling scheme.}
    \label{fig:roc-auc-beams}
\end{figure*}

\begin{table*}
\small
\centering
\begin{tabular}{lrrrrrrrrrrr}
\toprule
                         sampling  & \hspace{-1.7mm}$\varepsilon$ & \hspace{-1.7mm}count &  \hspace{-1.7mm}TPR@4.0 &  \hspace{-1.7mm}FNR@4.0 &  \hspace{-1.7mm}\shortstack{w/attck\\TPR@4.0} & \hspace{-1.7mm}\shortstack{w/attck\\FNR@4.0} &  \hspace{-1.5mm}TPR@5.0 &  \hspace{-1.7mm}FNR@5.0 & \hspace{-1.7mm}\shortstack{w/attck\\TPR@5.0} & \hspace{-1.7mm}\shortstack{w/attck\\FNR@5.0} \\\midrule
 m-nom. &            0.1 &    487 &            0.984 &            0.016 &                     0.819 &                     0.181 &            0.977 &            0.023 &                     0.577 &                     0.423 \\
         m-nom. &            0.3 &    487 &            0.984 &            0.016 &                     0.353 &                     0.647 &            0.977 &            0.023 &                     0.127 &                     0.873 \\
         m-nom. &            0.5 &    487 &            0.984 &            0.016 &                     0.094 &                     0.906 &            0.977 &            0.023 &                     0.029 &                     0.971 \\
         m-nom. &            0.7 &    487 &            0.984 &            0.016 &                     0.039 &                     0.961 &            0.977 &            0.023 &                     0.012 &                     0.988 \\
        beams &            0.1 &    489 &            0.998 &            0.002 &                     0.834 &                     0.166 &            0.998 &            0.002 &                     0.751 &                     0.249 \\
        beams &            0.3 &    489 &            0.998 &            0.002 &                     0.652 &                     0.348 &            0.998 &            0.002 &                     0.521 &                     0.479 \\
        beams &            0.5 &    489 &            0.998 &            0.002 &                     0.464 &                     0.536 &            0.998 &            0.002 &                     0.299 &                     0.701 \\
        beams &            0.7 &    489 &            0.998 &            0.002 &                     0.299 &                     0.701 &            0.998 &            0.002 &                     0.155 &                     0.845 \\
\bottomrule
\end{tabular}
\caption{Error rates for watermarked text before and after attack (w/attck) for generations of length $T=200\pm5$. For all settings we use $(\delta,\gamma)=(2.0,0.5)$.  Results are shown for both multinomial sampling and greedy 8-way beam search. The TPR and FNR rates without the attack are shown for reference, but they have no dependence on the attack budget $\varepsilon.$  For all experiments, no false positives were observed and so FPR$=0$ and TPR$=1$.}
\label{tab:attacked-acc}
\end{table*}

\section{Minor Variations}

\paragraph{Multiple Watermarks.} A company might also apply multiple watermarks to generated text, taking the union of all red lists at each token. This is a compromise in terms of watermark effectiveness, compared to a single watermark, however it allows additional flexibility. A company could run a public/private watermarking scheme, giving the public access to one of the watermarks to provide transparency and independent verification that text was machine-generated. At the same time, the company can keep the second watermark private and test text against both watermarks, to verify cases reported by the public watermark, or again to provide a stronger detection API. Such a setup would be especially effective in detecting whether an attack took place that attempted to remove the public watermark. 

\paragraph{Selective Watermarks in response to malicious activity}
Watermarks could also be used selectively. An API owner could turn on watermarking (or dial up its strength considerably via increased $\delta$) only when faced with suspicious API usage by some accounts, for example if a request appears to be part of malicious activity like creating synthetic tweets. This would give more leeway to benign API usages, but allow for improved tracing of malicious API utilization.

\paragraph{Discovering A Watermarking Scheme}
So far we assumed that an attacker is aware that a watermark is present. Could the attack discover this fact only by analyzing generated text? For a hard watermark, this would be easy: Some combinations of tokens will never be generated by the model, no matter how strongly they are prompted. Yet, for a soft watermark (especially with small $\delta$), that depends on, e.g. $h=10$ tokens via \cref{alg:private_robust_watermark}, this becomes harder. The attacker would need to distinguish the modification of green list logits via $\delta$ from naturally occurring biases of the LM. 

\section{Proof of Theorem \ref{maintheorem}}
We begin our proof with a useful lemma.
Using the spike entropy, we can predict how often a watermarked language model will spit out a green list token.  When the entropy is high, the language model has a lot of freedom and we expect the model to use green list tokens aggressively.  When the entropy is low, the model is more constrained and it is more likely to use a red list token.   

\begin{lemma} \label{tokenlemma}
Suppose a language model produces a raw (pre-watermark) probability vector $p\in (0,1)^N$.  Randomly partition $p$ into a green list of size $\gamma N$ and a red list of size $(1-\gamma) N$ for some $\gamma\in (0,1).$ Form the corresponding watermarked distribution by boosting the green list logits by $\delta$, as in Equation \eqref{logitboost}. Define $\alpha=exp(\delta).$ 

Sample a token index $k$ from the watermarked distribution.  The probability that the token is sampled from the green list is at least
$$ \prob[k\in G] \ge  \frac{\gamma\alpha}{ 1+(\alpha-1)\gamma} S\left(p,\frac{(1-\gamma)(\alpha - 1)}{ 1+(\alpha-1)\gamma}\right).$$
\end{lemma}

\begin{proof}
 When we add $\delta$ to the logits corresponding to the green list words, we increase their probabilities of being sampled. 
 We replace the raw probability $p_k$ for each green list word with the enlarged probability
   $$p^w_k \triangleq  \frac{\alpha p_k}{\sum_{i\in R} p_i + \alpha \sum_{i\in G} p_i },$$
where $G$ is the set of green list indices and $R$ is the complementary set of red list indices. We denote the sizes of these sets as $N_G$ and $N_R,$ respectively.

We begin our proof by bounding the size of a randomly chosen gre-list probability after it has been enlarged.  Consider the following process for creating the lists. First, choose a random entry $p_k$ and place it in the green list.  Then, randomly sample the remaining entries in the green list.  The expected value of a randomly chosen probability from the green list can be written  
 \begin{align}
 \expect_{k < N} \expect_{G, R}  \frac{\alpha p_k}{\sum_{i\in R} p_i + \alpha \sum_{i\in G} p_i }, \label{nested}
 \end{align}
 where the inner expectation is over uniformly random green/red partitions that satisfy $k\in G.$

 Now let's bound the inner expectation on the right.  Consider the helper function
 $$f_k(p) =  \expect_{G, R}  \frac{\alpha p_k}{\sum_{i\in R} p_i + \alpha \sum_{i\in G} p_i },$$
 where $G$ and $R$ are sampled at random from the set of partitions that satisfy $k \in G$. The value of $f_k$ is invariant to permutations in the order of the indices $\{p_i, i\neq k\}.$ For this reason $f(p) = \expect_\Pi f(\Pi p),$ where $\Pi$ is a random permutation that leaves $p_k$ in place.    Also, $f_k$ is convex in $p_{-k}$.  By Jensen's inequality, 
$$f(p) = \expect_\Pi f(\Pi p) \ge   f(\expect_\Pi \Pi p).  $$

The expectation on the right involves a probability vector  $\bar p \triangleq \expect_\Pi \Pi p$ in which $\bar p_i = (1-p_0)/(N-1)$ for $i\neq k.$  We now have
\begin{align} 
f_k(p) \ge f_k(\bar p) &=  \frac{\alpha p_k}{N_R(1-p_k)/(N-1) + \alpha (N_G-1)(1-p_0)/(N-1) + \alpha p_0 } \\
&=  \frac{\alpha p_k(N-1)}{ (N_R+\alpha N_G - \alpha) (1-p_k)   + \alpha p_0(N-1) } \\
&=  \frac{\alpha p_k(N-1)}{ N_R+\alpha N_G - \alpha +  (\alpha N - N_R-\alpha N_G  ) p_k  } \\ 
&=  p_k \frac{\alpha  N-\alpha }{ N_R+\alpha N_G - \alpha +  (\alpha N_R - N_R  ) p_k  } \\ 
&\ge  p_k \frac{\alpha  N}{ N_R+\alpha N_G +  (\alpha N_R - N_R  ) p_k  }.  \label{fbound}
\end{align}
In the last step we used the fact that the numerator is larger than the denominator, and so adding $\alpha$ to the numerator and denominator results in a small decrease in the bound.  Also, note that the fraction on the right side of \eqref{fbound} is strictly greater than 1 for any value of $p_k\in (0,1)$ and $\alpha \ge 1$. For this reason the bound is never vacuous, as $f_k(p)>p_k$. 

Now let $\gamma = N_G/N.$  This simplifies the notation of our intermediate result to
\begin{align}
f_k(p) &\ge \frac{\alpha p_k }{ (1-\gamma) +\alpha \gamma +  (\alpha - 1)(1-\gamma) p_k  }.
\end{align}

Using this expression to simplify \eqref{nested} we get
  $$\expect_{k < N} \expect_{G, R}  \frac{\alpha p_k}{\sum_{i\in R} p_i + \alpha \sum_{i\in G} p_i } = \expect_{k < N} f_k(p) 
  \ge \frac{\alpha N^{-1}}{  1+(\alpha-1)\gamma } S\left(p,\frac{(1-\gamma)(\alpha - 1)}{ 1+(\alpha-1)\gamma}\right).$$
The probability of sampling a token from the green list is exactly $N_G$ times larger than an average green list probability. The probability of sampling from the green list is thus given by
$$N_G \expect_{k < N} \expect_{G, R}  \frac{\alpha p_k}{\sum_{i\in R} p_i + \alpha \sum_{i\in G} p_i } 
\ge \frac{\gamma\alpha}{ 1+(\alpha-1)\gamma} S\left(p,\frac{(1-\gamma)(\alpha - 1)}{ 1+(\alpha-1)\gamma}\right).$$

  \end{proof}

It can be observed that the bound in Lemma \ref{tokenlemma} is never vacuous;  The probability of choosing a token from the green list is trivially at least $\gamma,$  and for any combination of finite logits the bound in Lemma \ref{tokenlemma} is strictly greater than this trivial lower bound. See the proof for a discussion of why.

Using this lemma, it's now fairly straightforward to prove the main theorem.

\begin{proof}
Lemma \ref{tokenlemma} bounds the probability of a single token being in the green list.  To compute the total number of green list tokens in the sequence, we simply sum this bound over all the tokens to get.  
 $$ \expect |s|_G  = \sum_t \frac{\gamma\alpha}{ 1+(\alpha-1)\gamma} S^{t}  = T \expect_t \frac{\gamma\alpha}{  1+(\alpha-1)\gamma } S^{t} 
 \ge 
 \frac{\gamma\alpha T}{ 1+(\alpha-1)\gamma} S^\star,$$
where $S^{(t)}$ represents the entropy of the distribution of token $t$.

To get the variance bound, we begin by noting that the variance of a Bernoulli random variable with success probability $p$ is $p(1-p).$   The expected number of green list tokens is a sum of independent random Bernoulli variables, each representing one token. These variables are {\em not} identically distributed, but rather each has a success probability given by Lemma \ref{tokenlemma}.  The variance of the sum is the sum of the variances, which is
$$ \text{Var}\,\, |s|_G =
\sum_t \frac{\gamma\alpha S^{(t)} }{ 1+(\alpha-1)\gamma}  \left(1-\frac{\gamma\alpha S^{(t)}}{  1+(\alpha-1)\gamma} \right)
= T \expect_t \frac{\gamma\alpha S^{(t)} }{  1+(\alpha-1)\gamma}  \left(1-\frac{\gamma\alpha S^{(t)} }{  1+(\alpha-1)\gamma} \right).  
$$
The expectation on the right contains a concave function of $S^{t}.$ By Jensen's inequality, we can pass the expectation inside the function to get
$$ \text{Var}\,\, |s|_G \le
T  \frac{\gamma\alpha  \expect_t S^{(t)} }{  1+(\alpha-1)\gamma}  \left(1-\frac{\gamma\alpha \expect_t S^{(t)} }{  1+(\alpha-1)\gamma } \right).  
$$
Finally, note that the probability of a token being in the green list is always at least $\gamma,$ regardless of the distribution coming from the language model.  Lemma \ref{tokenlemma} is never vacuous, and the success probability predicted by the Lemma is always at least $\gamma$.  If $\gamma\ge .5,$ then the variance of each Bernoulli trial is at most the variance of a Bernoulli trial with success probability $\gamma,$ which is given by $\gamma(1-\gamma).$  Plugging this into our bound gives  
 $$ \text{Var}\,\, |s|_G \le
T  \frac{\gamma\alpha  \expect_t S^{(t)} }{  1+(\alpha-1)\gamma}  \left(1-\frac{\gamma\alpha \expect_t S^{(t)} }{  1+(\alpha-1)\gamma} \right) \le T \gamma(1-\gamma).  
$$
\end{proof}

\section{Proof of Proposition \ref{perpbound}}
\begin{proof}
The probability of sampling token $k$ from the modified distribution is
   \begin{align}
   \hat p_k =  \expect_{G, R}  \frac{\alpha p_k}{\sum_{i\in R} p_i + \alpha \sum_{i\in G} p_i }, 
 \end{align}
where $G$ and $R$ are random partitions of the vocabulary indices. We can write this expected value as the sum of a contribution from the case in which $k\in G,$ and one in which $k\in R.$  We get 
   \begin{align} 
   \expect_{G, R}&  \frac{\alpha p_k}{\sum_{i\in R} p_i + \alpha \sum_{i\in G} p_i } = \expect_{G, R, k\in G}  \frac{\alpha p_k}{\sum_{i\in R} p_i + \alpha \sum_{i\in G} p_i }\\ & + \expect_{G, R, k\notin G}  \frac{\alpha p_k}{\sum_{i\in R} p_i + \alpha \sum_{i\in G} p_i } 
   \le  \gamma\alpha p_k +  (1-\gamma) p_k = (1+(\alpha-1)\gamma)p_k.
 \end{align}
The expected perplexity is then given by 
 $$  \expect_{G, R} \sum_k \hat p^{(t)}_k \ln(p^{(t)}_k) = \sum_k \expect_{G, R}  \hat p^{(t)}_k \ln(p^{(t)}_k) \le  (1+(\alpha-1)\gamma)p^*.$$
\end{proof}

\begin{table*}[h!]
\small
\caption{Performance measured using standard metrics Exact Match (EM) and whitespace-tokenized F1 score against each question's answer alias list. ``(W)'' indicates generation with the watermark. Data is 50,000 samples from the validation split of the unfiltered version of TriviaQA dataset. Questions are posed to the model in a zero-shot manner with no in-context demonstrations. Prompt template used: \texttt{f"The following is a trivia question with a single correct factual answer. Please provide the answer to the question.\textbackslash{}n\textbackslash{}nQuestion \{q\}\textbackslash{}n\textbackslash{}nAnswer: "}. Generation is performed using greedy decoding to maximize baseline/unwatermarked performance.}
\centering
\begin{tabular}{lcccccc}
\toprule
Model & EM & EM (W) & F1 & F1 (W) & z & z (W) \\
\midrule
\texttt{google/flan-ul2} & 0.374 & 0.336 & 0.415 & 0.378 & -0.007 & 0.402 \\
\texttt{bigscience/bloomz} & 0.296 & 0.259 & 0.343 & 0.312 & 0.008 & 0.255 \\
\bottomrule
\end{tabular}\label{tab:trivia-qa}
\end{table*}

\section{Impact of Watermarking on Model Factuality}\label{sec:factuality}

A key benefit of the soft watermarking scheme is that it (passively) adapts to the current entropy in the model's output distribution. If the model is highly confident on its next few token predictions, say those representing a specific named entity, then a soft watermark will not affect those predictions regardless of their factuality or groundedness. On the other hand, if the model is not confident on any particular tokens, then under standard decoding schemes, whether or not the final decoded output is hallucinatory will be a random event, and the watermark has an equal chance of upweighting tokens that result in more factual or more hallucinatory utterances.

To illustrate this, we present a small experiment that isolates this behavior. We take a model with reasonable competency in knowledge-intensive, closed-book question answering and evaluate its performance on the validation set of the TriviaQA question answering dataset \citep{JoshiTriviaQA2017}. \textit{Hypothesis:} since answers to factoid questions should be short, low entropy sequences, a soft watermark will yield low detection statistics, however, task performance will not degrade much under application of the watermark. The results for this experiment are shown in \cref{tab:trivia-qa} and provide some evidence that in factuality critical generation scenarios, a softly watermarked model is unlikely to deviate that much from its unwatermarked behavior (for better or worse). We observe less than a $4$ point drop in Exact Match performance under the application of a standard soft watermark ($\gamma,\delta=0.5,2.0$) for both Google's FLAN-UL2 model \citep{tay2022unifying} and Huggingface BigScience's BLOOMZ model \citep{muennighoff2022crosslingual}.

However, we note that this particular experimental setup is not a situation where we would actually deploy the watermark or expect it to work very well. Generating 5 to 10 tokens per question under greedy decoding and then testing those tokens for exact correctness, is something of a worst-case estimate on the cost of watermarking. In this scenario the prompt is highly constraining and the only things the watermark can do are either nothing, or directly cause the model to deviate from the argmax. Such deviations would be detrimental on any question where the model ``knows'' the correct answer but isn't overwhelmingly confident in the token sequence required to represent it (especially the surface form). We leave a more comprehensive study of the impacts of watermarking strategies on the factuality of LLMs in question answering and other knowledge intensive settings to future research.


\label{lastpageapp}
\end{document}